%% file: example.tex
\documentclass{article}

\usepackage[preprint]{corl_2026} 

\usepackage{graphicx}
\usepackage{wrapfig}
\usepackage{amsmath}
\usepackage{amssymb}
\usepackage{amsthm}
\usepackage{amsthm}
\usepackage{algorithm}
\usepackage{hyperref}
\usepackage{algorithmic}
\usepackage{adjustbox}
\usepackage{multirow}
\usepackage{booktabs}
\usepackage{caption}
\usepackage{enumitem}
\usepackage{placeins}
\newtheorem{lemma}{Lemma}

\newcommand{\method}{MIRROR}
\newcommand{\norm}[1]{\left\lVert#1\right\rVert}

\setlist[itemize]{itemsep=0pt,parsep=0pt,topsep=0pt,partopsep=0pt}

\title{\method{}: Reflecting Hidden Symmetries from Offline Robot Interaction Data}

\author{
  Shiqi Liu$^{1}$ \quad
  Yihang Yao$^{1}$ \quad
  Peide Huang$^{1}$ \quad
  Mengdi Xu$^{2}$ \quad
  Ding Zhao$^{1}$ \\
  $^{1}$Carnegie Mellon University \quad $^{2}$Stanford University
}

\hypersetup{
  pdftitle={MIRROR: Reflecting Hidden Symmetries from Offline Robot Interaction Data},
  pdfauthor={Shiqi Liu, Yihang Yao, Peide Huang, Mengdi Xu, Ding Zhao}
}

\begin{document}
\maketitle


\input{sections/abstract}


\input{sections/introduction}
\input{sections/related_works}
\input{sections/preliminary}
\input{sections/method}
\input{sections/experiments}
\input{sections/sim2real}
\input{sections/conclusion}
\input{sections/limitations}


\clearpage
\acknowledgments{We thank Haohong Lin and William Jongwon Han for their constructive feedback and informative discussions.}


\bibliography{example}  

\clearpage
\appendix
\input{appendix/overview}
\input{appendix/prove}
\input{appendix/env_description}

\input{appendix/implementation_details}
\input{appendix/evaluation_all_envs}
\input{appendix/equivariant_bc}
\input{appendix/sim2real}

\end{document}

%% file: sections/abstract.tex
\begin{figure}[h]
  \centering
  \includegraphics[width=1.0\linewidth]{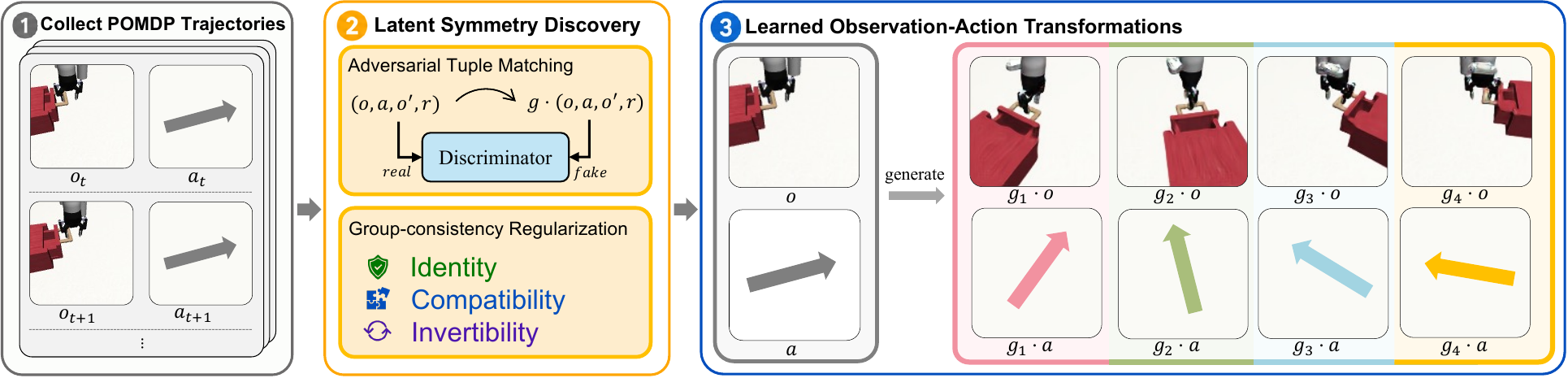}
  \caption{\textbf{Overview of \method{}.} From offline robot interaction data, \method{} learns how a specified symmetry group jointly transforms observations and actions. Adversarial matching encourages transformed transition tuples to preserve the original data distribution, while group-consistency regularization encourages identity, composition, and inverse consistency. The examples on the right illustrate the learned coupling: transforming a scene also transforms its associated robot command, mapping an observation--action pair $(o,a)$ to corresponding pairs $(g\cdot o,g\cdot a)$.
  }
  \label{fig:teaser}
\end{figure}

\begin{abstract}
  Symmetry-aware robot policy learning requires coupled observation--action transformations. For pixel observations, even simple physical symmetries can induce complex, viewpoint-dependent image transformations that are difficult to specify by hand. Existing symmetry-aware methods often rely on representations with predefined transformations, requiring preprocessing tailored to the sensing setup and task. We introduce \method{}, a framework that learns how a known symmetry group operates jointly on observations and actions in partially observable Markov decision processes (POMDPs) from offline interaction data. \method{} couples these transformations through adversarial tuple matching over one-step transition tuples, while explicit group-consistency regularization penalizes violations of the group-operation axioms. On four Robosuite tabletop manipulation tasks, \method{} achieves the lowest aggregate error on three of four symmetry-quality metrics, including a $\mathbf{61\%}$ reduction in compatibility error---which measures consistency under transformation composition---relative to an adapted LaLiGAN baseline. Using the learned transformations in frame-averaged behavior cloning (BC) yields higher mean success rates than vanilla BC on all four tasks, increasing average success from $80.48\%$ to $85.12\%$, with per-task gains up to $\mathbf{9.66}$ percentage points. We also repeat symmetry discovery and policy training in simulation with heightmap and segmentation-mask observations, then deploy the resulting policy zero-shot on a real Kinova GEN3 arm. Project page: \url{https://pomdp-mirror.github.io}.
\end{abstract}

%% file: sections/introduction.tex
\section{Introduction}

Symmetry can improve sample efficiency and generalization in robot policy learning by allowing experience to be shared across equivalent task configurations~\citep{wang2022so2, liu2023continual, ryu2023equivariant, huang2024fourier, eisner2024deep, yang2024equibot, wang2024equivariant, gao2024riemann, hu20253d, klee2026raven}. Exploiting this structure requires coupled observation--action transformations, which existing symmetry-aware methods typically assume are known. This assumption is difficult to satisfy when the robot observes the scene through images. For example, rotating a drawer changes both its appearance to a fixed camera and the direction needed to pull it open. Although the physical rotation is simple, the resulting image transformation depends on the scene geometry and viewpoint and generally cannot be expressed as an image rotation. Existing methods address this difficulty by using point clouds~\citep{ryu2023equivariant, huang2024fourier, eisner2024deep, yang2024equibot, wang2024equivariant, gao2024riemann} or projecting observations into representations with known group operations~\citep{wang2022so2, liu2023continual, hu20253d, klee2026raven}. These representations make the group operation explicit, but constructing them can require preprocessing tailored to the sensing setup and task. Extending the same approach to a new observation modality can therefore require redesigning how observations are represented and transformed. Thus, even when the relevant symmetry group is known, exploiting it still depends on manually specifying suitable representations and the corresponding observation--action transformations.

This motivates a natural question: can we learn the required representations and coupled observation--action transformations directly from environment interaction data? Existing symmetry-discovery methods provide a promising foundation; for example, LaLiGAN~\citep{yang2023latent} learns nonlinear symmetries of high-dimensional data through latent representations. Applying these approaches to visual control introduces two requirements. First, observations and actions must transform together to preserve their relationships with subsequent observations and rewards. Second, the learned transformations must satisfy the group-operation axioms: applying two transformations successively should agree with applying their composition. Distribution matching alone does not ensure this consistency, and even an exact linear group operation in latent space can lose it after approximate encoding and decoding. These requirements motivate learning coupled observation--action transformations while explicitly regularizing their group consistency.

We introduce \textbf{\method}, a framework that learns how a known symmetry group $G$ operates jointly on observations and actions from offline interaction data in partially observable Markov decision processes (POMDPs). The core idea is to match the distributions of original and transformed one-step transition tuples $\tau=(o,a,o',r)$, coupling observation and action learning through the observed transitions and rewards. We implement this distribution matching with an adversarial objective, which we call \emph{adversarial tuple matching}. Identity, compatibility, and invertibility regularizers penalize violations of the group-operation axioms. To represent nonlinear transformations of pixel observations, \method{} first compresses images with a pretrained, frozen perceptual autoencoder. Separate encoder--decoder pairs then map these perceptual latents and the actions to symmetry latent spaces with fixed linear representations of $G$. Both objectives operate on tuples of perceptual latents, actions, and rewards. Learning these mappings from each task's data allows the same discovery procedure to be used across tasks without manually specifying their coupled observation--action transformations.

We evaluate \method{} on four Robosuite~\citep{robosuite2020} tabletop manipulation tasks with RGB observations. It achieves the lowest aggregate error on three of four symmetry-quality metrics, including a $61\%$ reduction in compatibility error relative to our adapted LaLiGAN~\citep{yang2023latent} baseline. Incorporating the discovered transformations into frame-averaged behavior cloning (BC) yields higher mean success rates than vanilla BC on all four tasks. Average success increases from $80.48\%$ to $85.12\%$, with the largest per-task gain of $9.66$ percentage points on Cube Dropping. We also apply the framework to heightmap and segmentation-mask observations, repeating symmetry discovery and policy training in simulation before deploying the resulting policy zero-shot on a Kinova GEN3 arm.

Our contributions:
\begin{itemize}[itemsep=0pt, parsep=0pt, topsep=0pt, partopsep=0pt]
  \item \textbf{Adversarial symmetry discovery in POMDPs:} We formulate the discovery of coupled observation--action transformations as adversarial tuple matching over one-step transition tuples $(o, a, o', r)$.
  \item \textbf{Explicit group-consistency regularization:} We penalize violations of identity, compatibility, and invertibility, improving aggregate compatibility and invertibility errors over the adapted GAN baselines.
  \item \textbf{Policy learning with discovered symmetries:} We use the learned transformations in a frame-averaged BC policy, improving simulation success rates by up to $9.66$ percentage points over vanilla BC. We also demonstrate the framework with a second observation modality and zero-shot policy deployment on a Kinova GEN3 arm.
\end{itemize}

%% file: sections/related_works.tex
\section{Related Work}
\label{sec:related_works}

\paragraph{Exploiting symmetry in robot policy learning.}
Symmetry has been extensively studied as an inductive bias for robotic manipulation, yielding consistent gains in sample efficiency and generalization~\citep{wang2022so2, liu2023continual, ryu2023equivariant, huang2024fourier, eisner2024deep, yang2024equibot, wang2024equivariant, gao2024riemann, hu20253d, klee2026raven}. A central limitation is that these methods assume the way the symmetry group operates on observations and actions is known \emph{a priori}, which rarely holds for visual manipulation tasks. Existing work circumvents this with task-specific observation preprocessing, such as using point clouds from depth cameras~\citep{ryu2023equivariant, huang2024fourier, eisner2024deep, yang2024equibot, wang2024equivariant, gao2024riemann} or projecting observations into a representation where the symmetry operates linearly. Concrete instances of the latter include projecting depth images to a top-down heightmap~\citep{wang2022so2}, inpainting images to make the top-down view equivariant~\citep{liu2023continual}, mapping fisheye images to a 3D sphere~\citep{hu20253d}, and lifting 2D images to 3D ray features~\citep{klee2026raven}. No single recipe transfers across environments.
\method{} addresses this limitation by learning coupled nonlinear observation--action transformations for a specified group $G$ from offline interaction data, reducing dependence on manually designed, modality-specific transformations.

\paragraph{Data-driven symmetry discovery.}
Automatic symmetry discovery identifies unknown symmetries directly from data~\citep{dehmamy2021automatic, moskalev2022liegg, hu2025explicit} and has improved downstream tasks such as equation discovery and long-horizon forecasting~\citep{yang2023liegan, yang2023latent}.
However, two limitations restrict prior work's applicability to control in POMDPs.
First, prior methods discover symmetries on observations alone, whereas decision-making requires \emph{coupled} observation--action transformations that preserve the transition structure.
Second, prior methods either omit explicit group-consistency regularization entirely~\citep{desai2022symmetry, liu2025physically, yang2023liegan} or regularize only a subset of the group-operation axioms, as in LaLiGAN~\citep{yang2023latent}, which regularizes only the identity axiom and accordingly violates the other group-operation axioms in practice.
To address both gaps, we discover coupled observation--action transformations from one-step transition tuples $\tau=(o,a,o',r)$ via adversarial tuple matching, and we explicitly regularize the group-operation axioms (identity, compatibility, invertibility) to encourage the learned transformations to behave as valid group operations.

%% file: sections/preliminary.tex
\section{Preliminary}

\paragraph{Partially Observable Markov Decision Processes.}
A POMDP $\mathcal M=(\mathcal S,\mathcal A,\mathcal O,P,\Omega,R,\mu_0)$ has measurable state, action, and observation spaces $\mathcal S,\mathcal A,\mathcal O$, respectively. Let $\mathcal P(\mathcal X)$ denote the set of probability measures on $\mathcal X$. The maps $P:\mathcal S\times\mathcal A\to\mathcal P(\mathcal S)$ and $\Omega:\mathcal S\to\mathcal P(\mathcal O)$ are the transition and observation kernels, respectively. The initial state distribution is $\mu_0\in\mathcal P(\mathcal S)$, and $R:\mathcal S\times\mathcal A\to\mathbb R^{d_r}$ is a measurable $d_r$-dimensional reward function.

\paragraph{Groups, group operations, and representations.}
A \emph{group} is a set $G$ with an associative binary operation $\circ$, an identity element $e$, and inverses. A \emph{group operation}\footnote{We use the term \emph{group operation} (rather than the more common \emph{group action}) to avoid confusion with the \emph{actions} in a POMDP.} of $G$ on a set $\mathcal{X}$ is a map $\phi : G \times \mathcal{X} \to \mathcal{X}$ satisfying three axioms: \emph{identity}, $\phi(e, x) = x$; \emph{compatibility}, $\phi(g_2, \phi(g_1, x)) = \phi(g_2 \circ g_1, x)$; and \emph{invertibility}, $\phi(g^{-1}, \phi(g, x)) = x$; we sometimes write $g \cdot x := \phi(g, x)$. A \emph{linear representation} of $G$ on a vector space $V$ is a homomorphism $\rho : G \to \mathrm{GL}(V)$ with $\rho(g_2 \circ g_1) = \rho(g_2)\rho(g_1)$ and $\rho(e) = I$.

\paragraph{Symmetric POMDPs.}
A POMDP is \emph{$G$-symmetric} if measurable group operations $\phi_s,\phi_a,\phi_o$ on $\mathcal S,\mathcal A,\mathcal O$, respectively, make its transition and observation kernels equivariant and its reward invariant: $P(g\cdot B_s\mid g\cdot s,g\cdot a) = P(B_s\mid s,a)$, $\Omega(g\cdot B_o\mid g\cdot s) = \Omega(B_o\mid s)$, and $R(g\cdot s,g\cdot a) = R(s,a)$, for all $g\in G$, $s\in\mathcal S$, $a\in\mathcal A$. Invariance of the offline transition-tuple distribution additionally depends on data collection.

\paragraph{Observable symmetry in POMDPs.}
Given group operations $\phi_o, \phi_a$ of $G$ on observations and actions, the induced transformation on one-step transition tuples $\tau = (o, a, o', r)$ is $\Phi(g, \tau) := \big(\phi_o(g, o),\ \phi_a(g, a),\ \phi_o(g, o'),\ r\big)$.
The POMDP has \emph{observable symmetry} under $G$ if the resulting transition-tuple distribution is invariant: $p(\tau) = p(\Phi(g, \tau))$ for all $g \in G$, where $p(\tau)$ is the distribution sampled from the offline dataset $\mathcal D$.

%% file: sections/method.tex
\section{Methodology}
\label{sec:method}

\begin{figure}[htbp]
  \centering
  \includegraphics[width=1.0\linewidth]{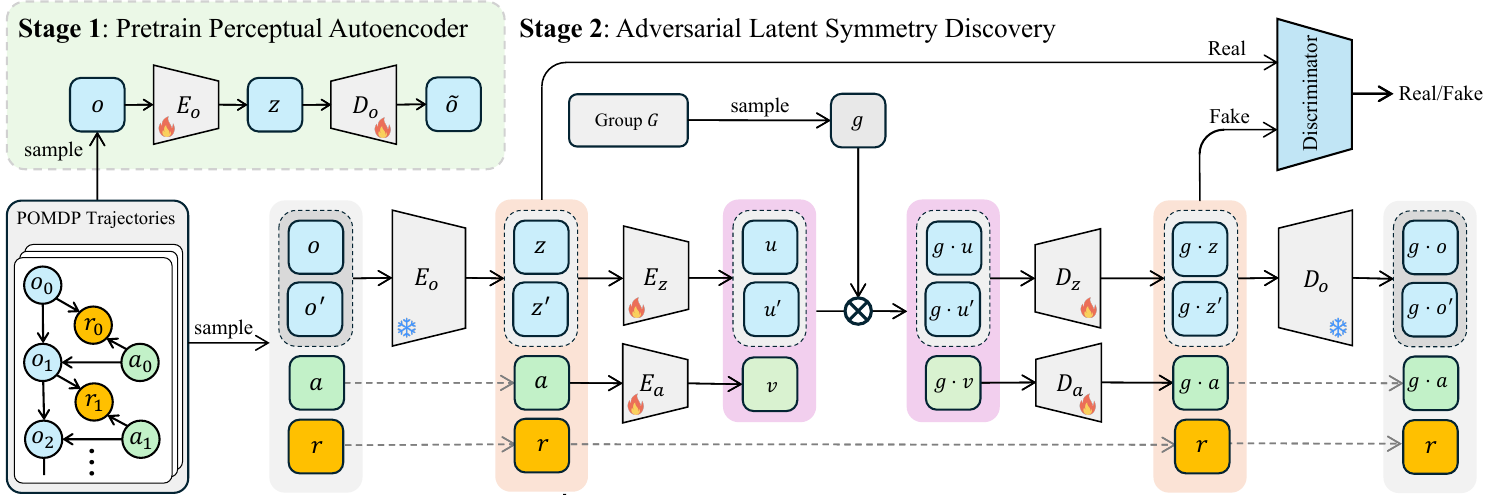}
  \caption{
    \textbf{Two-stage latent symmetry discovery.} Stage~1 trains a perceptual autoencoder $(E_o,D_o)$ to reconstruct observations and then freezes it. In Stage~2, separate symmetry modules encode the perceptual latents $(z,z')$ and action $a$ into symmetry spaces, apply fixed linear representations of the same sampled $g\in G$, and decode the transformed outputs. Adversarial training uses a tuple discriminator to match original and transformed latent transition-tuple distributions, with rewards unchanged. Group-consistency regularization penalizes violations of identity, compatibility, and invertibility.
  }
  \label{fig:method_pipeline}
\end{figure}

Given a known symmetry group $G$ and an offline dataset $\mathcal D=\{(o,a,o',r)\}$ from a POMDP, our goal is to learn coupled observation--action transformations $(\phi_o,\phi_a)$ that preserve the transition-tuple distribution and satisfy the group-operation axioms. Joint tuple matching encourages the transformed observations and actions to remain consistent with the recorded transitions and rewards. Figure~\ref{fig:method_pipeline} summarizes \method{}; Section~\ref{sec:param} describes the latent parameterization, and Section~\ref{sec:learning} presents the learning objective.

\subsection{Latent group-operation parameterization}
\label{sec:param}

\paragraph{Stage~1: Perceptual autoencoder.}
To reduce the cost of symmetry discovery from pixel observations, we pretrain a perceptual autoencoder $(E_o,D_o)$ on $\mathcal D$ for reconstruction and then freeze it (Appendix~\ref{appendix:implementation}). The encoder provides a compact perceptual latent $z=E_o(o)\in\mathbb R^{d_z}$ on which the subsequent symmetry modules operate~\citep{rombach2022high,van2017neural,razavi2019generating,esser2021taming}.
Appendix~\ref{appendix:prove} gives the conditions under which latent group operations preserve the group-operation axioms in observation space.

\paragraph{Stage~2: Encode--transform--decode template for group operations.}
Following~\citet{yang2023latent}, we parameterize observation and action group operations using separate encode--transform--decode modules. The observation symmetry module $(E_z,D_z)$ encodes the Stage~1 perceptual latent $z=E_o(o)$ into an observation symmetry latent $u=E_z(z)\in\mathbb R^{d_u}$, while the action symmetry module $(E_a,D_a)$ encodes $a$ into an action symmetry latent $v=E_a(a)\in\mathbb R^{d_v}$. The fixed linear representations $\rho_z(g)$ and $\rho_a(g)$ act on $u$ and $v$, respectively. Decoding the transformed codes gives $\phi_z(g,z):=D_z(\rho_z(g)u)$ and $\phi_a(g,a):=D_a(\rho_a(g)v)$, returning to the perceptual latent space and action space, respectively. The frozen perceptual decoder then induces the observation-space operation $\phi_o(g,o):=D_o(\phi_z(g,E_o(o)))$.

\subsection{Learning latent symmetry from offline transition tuples}
\label{sec:learning}

\paragraph{Adversarial tuple matching in latent space.}
We learn coupled observation--action transformations by matching the joint distribution of transition tuples. This encourages the transformed observations and actions to remain consistent with the recorded transitions and rewards. For each transition tuple $\tau=(o,a,o',r)$, we form the latent transition tuple $t=(z,a,z',r)$, where $z=E_o(o)$ and $z'=E_o(o')$. Using the same sampled $g\in G$, we transform both $z$ and $z'$ with $\phi_z$ and the action with $\phi_a$, leaving the reward unchanged. This defines the latent tuple map $\Phi_z(g,t):=(\phi_z(g,z),\phi_a(g,a),\phi_z(g,z'),r)$. A discriminator $f_\psi$ learns to distinguish original and transformed tuples, while the symmetry modules $\Theta=(E_z,D_z,E_a,D_a)$ learn to match their distributions through the following adversarial tuple-matching objective:
\begin{equation}
  \mathcal L_{\mathrm{adv}}(\Theta, \psi)\ :=\ \mathbb E_{t}\!\big[f_\psi(t)\big]\ -\ \mathbb E_{t,\,g}\!\big[f_\psi(\Phi_z(g, t))\big],
  \label{eq:adv_critic}
\end{equation}
optimized as $\min_\Theta \max_\psi \mathcal L_{\mathrm{adv}}$. Implementation details are provided in Appendix~\ref{appendix:implementation}.

\paragraph{Group-consistency regularization.}
Although the fixed linear representations satisfy the representation identities exactly, learned encoding and decoding can introduce consistency errors. We therefore use three regularizers. \emph{Identity} encourages the identity element to leave a tuple unchanged. \emph{Compatibility} encourages successive transformations to agree with their composition. \emph{Invertibility} encourages a transformation followed by its inverse to recover the original tuple. The corresponding losses are
\begin{equation}
  \begin{aligned}
    \mathcal L_{\mathrm{id}}   & = \mathbb E_{t}\!\left[\norm{t-\Phi_z(e,t)}_2^2\right],                                            \\[-3pt]
    \mathcal L_{\mathrm{comp}} & = \mathbb E_{t,g_1,g_2}\!\left[\norm{\Phi_z(g_2,\Phi_z(g_1,t))-\Phi_z(g_2\circ g_1,t)}_2^2\right], \\[-3pt]
    \mathcal L_{\mathrm{inv}}  & = \mathbb E_{t,g}\!\left[\norm{\Phi_z(g^{-1},\Phi_z(g,t))-t}_2^2\right].
  \end{aligned}
  \label{eq:axiom_losses}
\end{equation}
\paragraph{Discouraging trivial solutions.}
Distribution matching and group consistency also admit a trivial solution: every group element leaves all observations and actions unchanged. Following LaLiGAN~\citep{yang2023latent}, we discourage this solution with a nontriviality penalty $\mathcal L_{\mathrm{nt}}$ based on the cosine similarity between original and transformed perceptual latents.

\paragraph{Overall objective.}
We train the symmetry modules to match original and transformed latent transition-tuple distributions, satisfy the group-operation axioms, and avoid trivial identity transformations. These goals give the weighted objective:
\begin{equation}
  \mathcal L_{\mathrm{total}}(\Theta;\psi)
  \ :=\
  \underbrace{\lambda_{\mathrm{adv}}\mathcal L_{\mathrm{adv}}}_{\text{adversarial tuple matching}}
  +\underbrace{\lambda_{\mathrm{id}}\mathcal L_{\mathrm{id}}+\lambda_{\mathrm{comp}}\mathcal L_{\mathrm{comp}}+\lambda_{\mathrm{inv}}\mathcal L_{\mathrm{inv}}}_{\text{group consistency}}
  +\underbrace{\lambda_{\mathrm{nt}}\mathcal L_{\mathrm{nt}}}_{\text{nontriviality}}.
  \label{eq:full_obj}
\end{equation}
We alternate discriminator and symmetry-module updates (Appendix~\ref{appendix:implementation}). After training, we obtain the coupled observation--action transformations $\phi_o$ and $\phi_a$.

%% file: sections/experiments.tex
\section{Experiments}
\label{sec:experiments}
\label{sec:exp_setup}
\begin{wrapfigure}{R}{0.30\textwidth}
  \centering
  \vspace{-0.5em}
  \includegraphics[width=\linewidth]{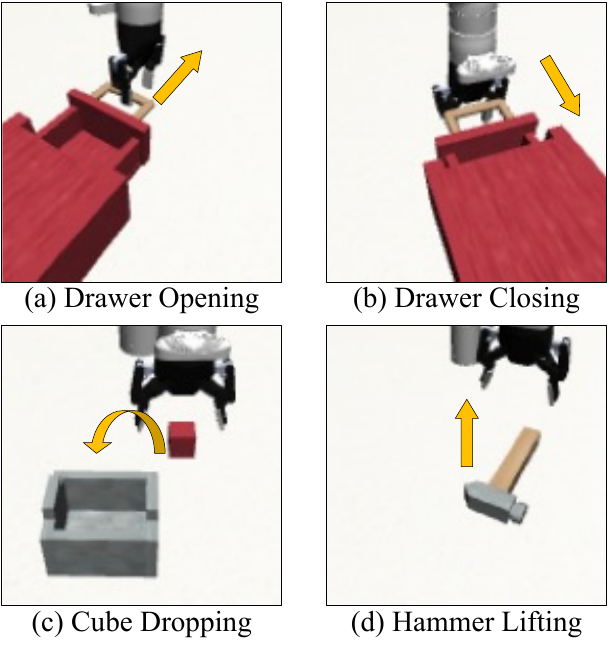}
  \caption{\textbf{Simulated manipulation tasks.} The four Robosuite tasks use RGB observations from a fixed camera. Arrows indicate the intended object motion.}
  \label{fig:env_figure_view}
  \vspace{-2em}
\end{wrapfigure}
We first visualize the learned transformations, then compare symmetry quality against baselines and ablations, and finally evaluate downstream policy learning. We study planar rotational symmetry ($G=SO(2)$) on four custom Robosuite~\citep{robosuite2020} tasks: \textbf{Drawer Opening}, \textbf{Drawer Closing}, \textbf{Cube Dropping}, and \textbf{Hammer Lifting} (Figure~\ref{fig:env_figure_view}).

For each task, we collect $100{,}000$ training and $10{,}000$ evaluation transition tuples using a finite-state-machine controller. Observations are RGB images from a fixed camera, and actions are end-effector commands normalized to $[-1,1]$. Action coordinates, rewards, and data collection details are given in Appendix~\ref{appendix:env_description}.

\subsection{What does \method{} discover?}
\label{sec:exp_q1}

\method{} learns how to change a scene and its associated robot command together. In Figure~\ref{fig:env_symmetry_vis}, the drawer's appearance changes jointly with the planar motion command $(\Delta x,\Delta y)$, while vertical translation $\Delta z$ and yaw $\Delta\psi$ remain approximately unchanged. This illustrates the relationship motivated in the introduction: changing the task configuration also requires changing the robot's motion.

\begin{figure}[htbp]
  \includegraphics[width=0.97\linewidth]{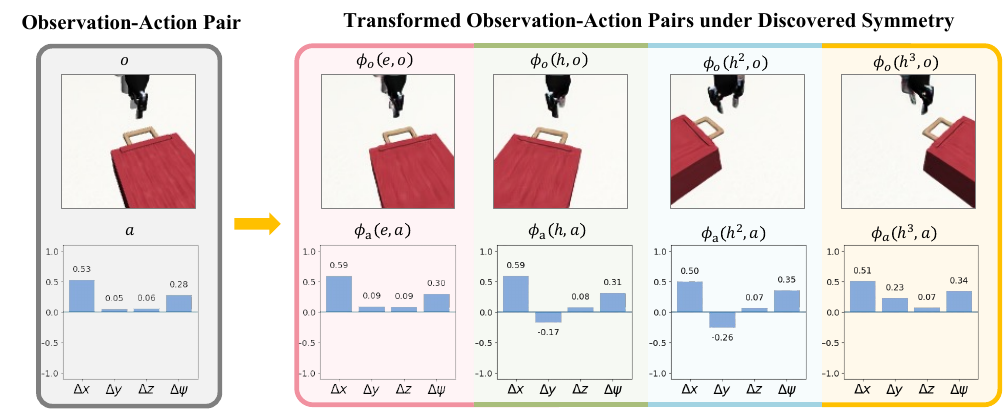}
  \caption{
    \textbf{\method{} learns coupled changes to the scene and robot command.}
    In this Drawer Opening example, the original pair $(o,a)$ is shown on the left, followed by four learned pairs $(\phi_o(g,o),\phi_a(g,a))$ for $g\in C_4=\{e,h,h^2,h^3\}$.
    The images show the scene configurations, and the bars show the corresponding normalized commands on a $[-1,1]$ scale.
    The columns use latent rotation angles $0^\circ,90^\circ,180^\circ,270^\circ$; the decoded transformations need not be geometric rotations by these angles.
  }
  \label{fig:env_symmetry_vis}
\end{figure}

The image and command transformations are learned from offline interaction tuples for a specified group. Similar coupling appears across all four tasks (Appendix~\ref{appendix:qualitative_all_envs}). We next evaluate whether these transformations preserve the data distribution and satisfy the group-operation axioms.

\subsection{How accurately does \method{} recover observation--action symmetries?}
\label{sec:exp_q2}
\label{sec:exp_ablation}

\paragraph{Symmetry-quality metrics.}
We evaluate the learned transformations with four complementary checks. \textbf{Symmetry invariance error} measures how closely transformed transition tuples match the original data distribution. \textbf{Identity error} checks that applying the identity leaves the input unchanged. \textbf{Compatibility error} measures composition consistency: applying two transformations successively should agree with applying their combination in one step. \textbf{Invertibility error} checks that a transformation followed by its inverse recovers the input. Lower is better for all four metrics; full definitions are in Appendix~\ref{sec:eval}.

\paragraph{Baselines and ablations.}
We adapt two GAN-based baselines to the same tuple formulation and architecture.
Within this shared setting, \textbf{SymGAN} uses adversarial tuple matching~\citep{desai2022symmetry,yang2023liegan,liu2025physically}, while \textbf{LaLiGAN} additionally regularizes identity consistency~\citep{yang2023latent}.
Neither adapted baseline includes compatibility or invertibility regularization.
To assess the contribution of each loss term, we remove the adversarial tuple-matching, identity, compatibility, or invertibility term from \method{} one at a time.
We use three training seeds per task and the evaluation angles specified in Appendix~\ref{sec:eval}. Figure~\ref{fig:symmetry_error_summary} summarizes all four tasks; per-task results are in Appendix~\ref{appendix:additional_env_results}.

\begin{figure}[htbp]
  \centering
  \begin{minipage}{0.25\linewidth}
    \centering
    \includegraphics[width=\linewidth]{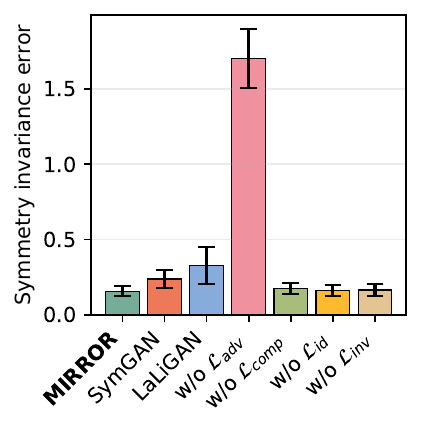}\\
    \vspace{-0.5em}
    \small (a) Symmetry invariance
  \end{minipage}\hfill
  \begin{minipage}{0.25\linewidth}
    \centering
    \includegraphics[width=\linewidth]{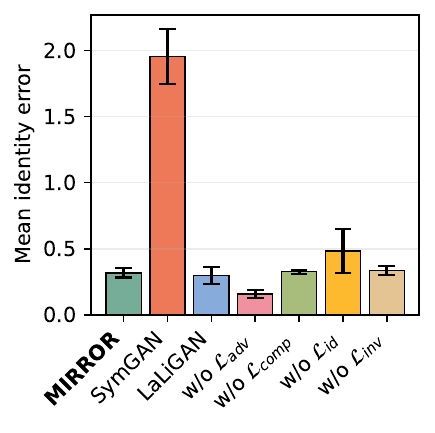}\\
    \vspace{-0.5em}
    \small (b) Identity
  \end{minipage}\hfill
  \begin{minipage}{0.25\linewidth}
    \centering
    \includegraphics[width=\linewidth]{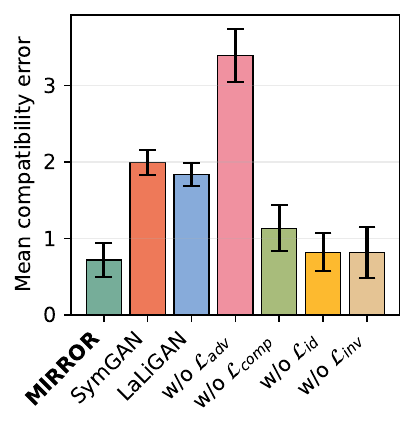}\\
    \vspace{-0.5em}
    \small (c) Compatibility
  \end{minipage}\hfill
  \begin{minipage}{0.25\linewidth}
    \centering
    \includegraphics[width=\linewidth]{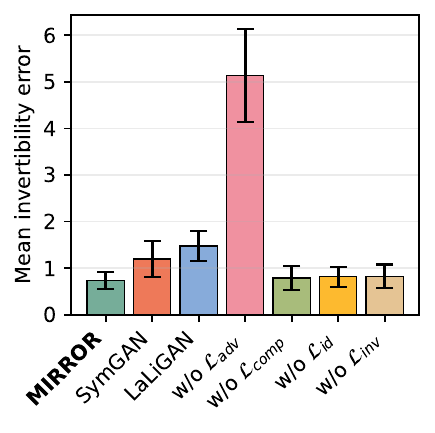}\\
    \vspace{-0.5em}
    \small (d) Invertibility
  \end{minipage}
  \caption{
    \textbf{Symmetry quality across four manipulation tasks.}
    Bars show mean errors and error bars indicate $\pm 1$ standard deviation across twelve task--seed results (four tasks, three seeds each), after averaging angle-dependent errors within each run.
    The ``w/o'' variants remove one loss term from \method{}. Lower is better for all metrics.
    \method{} has the lowest aggregate symmetry invariance, compatibility, and invertibility errors.
  }
  \label{fig:symmetry_error_summary}
\end{figure}

\paragraph{Distribution preservation.}
\method{} achieves the lowest aggregate symmetry invariance error (Figure~\ref{fig:symmetry_error_summary}(a); Table~\ref{tab:symmetry_quality}), supported qualitatively by increasing overlap of original and transformed latent data during training (Appendix~\ref{appendix:qualitative_all_envs}). Removing adversarial tuple matching raises this error from $0.157$ to $1.701$: group-consistency regularization alone does not preserve the interaction-data distribution in our experiments. This ablation also produces the largest aggregate compatibility and invertibility errors, despite lower identity error.

\begin{table}[htbp]
  \centering
  \caption{\textbf{Symmetry-quality comparison and loss ablations} across four tasks (mean $\pm$ standard deviation over twelve task--seed results; lower is better). Rows below the divider remove one loss term from Eq.~\eqref{eq:full_obj} relative to \method{}. Bold indicates the lowest error in each column across all rows.
  }
  \label{tab:symmetry_quality}
  \setlength{\tabcolsep}{6pt}
  \renewcommand{\arraystretch}{1.1}
  \resizebox{0.8\columnwidth}{!}{%
    \begin{tabular}{l|cccc}
      \toprule
      Method                            & Symmetry invariance        & Identity                   & Compatibility              & Invertibility              \\
      \midrule
      \method{} (ours)                  & $\mathbf{0.157 \pm 0.035}$ & $0.317 \pm 0.034$          & $\mathbf{0.717 \pm 0.221}$ & $\mathbf{0.730 \pm 0.183}$ \\
      SymGAN                            & $0.237 \pm 0.061$          & $1.954 \pm 0.206$          & $1.994 \pm 0.167$          & $1.198 \pm 0.389$          \\
      LaLiGAN                           & $0.329 \pm 0.122$          & $0.296 \pm 0.066$          & $1.838 \pm 0.155$          & $1.476 \pm 0.314$          \\
      \midrule
      w/o $\mathcal{L}_{\mathrm{adv}}$  & $1.701 \pm 0.195$          & $\mathbf{0.156 \pm 0.031}$ & $3.394 \pm 0.344$          & $5.131 \pm 0.995$          \\
      w/o $\mathcal{L}_{\mathrm{id}}$   & $0.162 \pm 0.037$          & $0.483 \pm 0.166$          & $0.818 \pm 0.251$          & $0.814 \pm 0.211$          \\
      w/o $\mathcal{L}_{\mathrm{comp}}$ & $0.174 \pm 0.036$          & $0.325 \pm 0.016$          & $1.134 \pm 0.301$          & $0.788 \pm 0.250$          \\
      w/o $\mathcal{L}_{\mathrm{inv}}$  & $0.164 \pm 0.038$          & $0.337 \pm 0.036$          & $0.813 \pm 0.333$          & $0.826 \pm 0.253$          \\
      \bottomrule
    \end{tabular}%
  }
\end{table}

\paragraph{Group consistency.}
Both GAN baselines exhibit substantially higher composition errors: aggregate compatibility error is $1.994$ for SymGAN and $1.838$ for LaLiGAN, versus $0.717$ for \method{}---a reduction of at least $61\%$ (Figure~\ref{fig:symmetry_error_summary}(c); Table~\ref{tab:symmetry_quality}). LaLiGAN's slightly lower identity error does not translate into consistent composition: reconstructing the input under the identity is insufficient to ensure that different transformations work together. The full model also reduces aggregate invertibility error by at least $39\%$ relative to either GAN baseline (Figure~\ref{fig:symmetry_error_summary}(d)).

The loss ablations reveal an asymmetry among the group-consistency regularizers. Removing $\mathcal L_{\mathrm{comp}}$, while retaining identity and invertibility regularization, increases aggregate compatibility error from $0.717$ to $1.134$. Removing identity or invertibility regularization produces smaller increases, to $0.818$ and $0.813$, respectively (Table~\ref{tab:symmetry_quality}).

\begin{figure}[!b]
  \centering
  \includegraphics[width=0.95\linewidth]{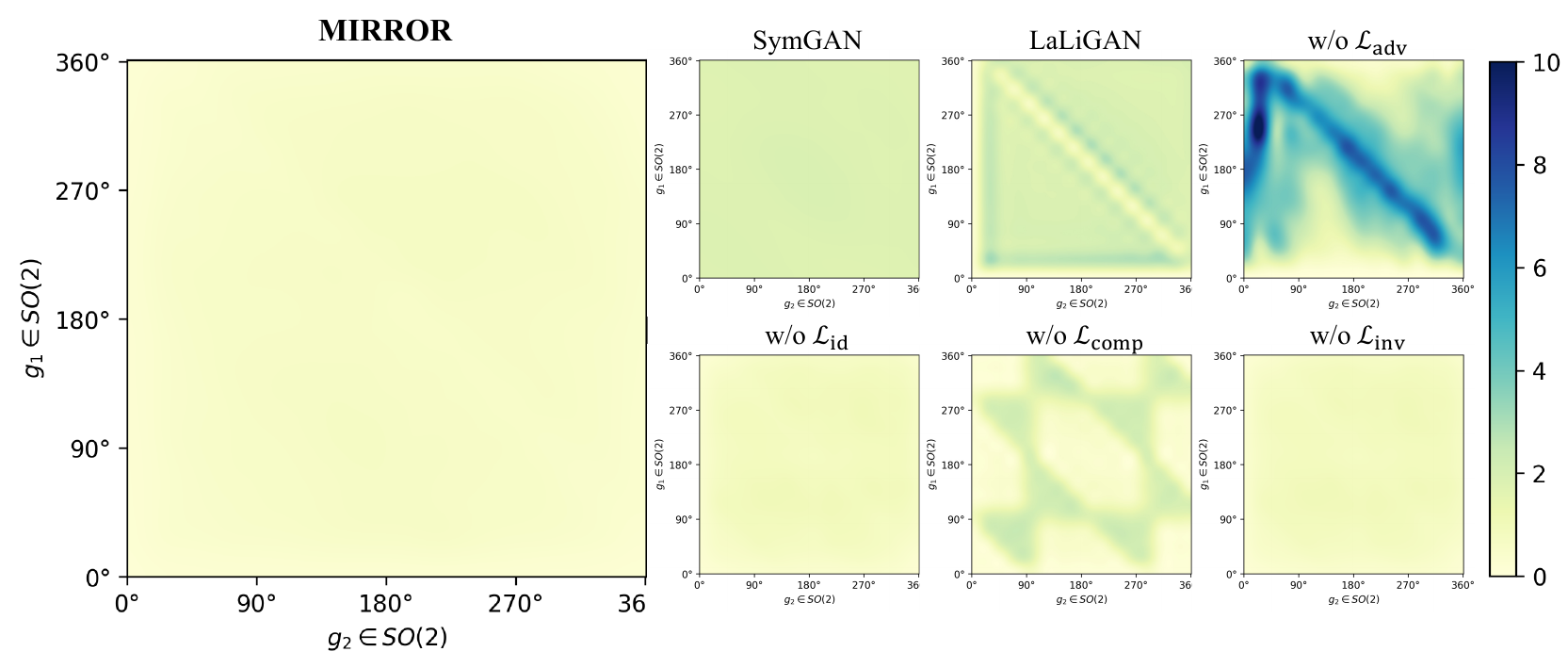}
  \caption{
    \textbf{Compatibility error heatmaps on Drawer Opening.}
    Errors compare successive transformations with their composition at all $256$ ordered pairs of $16$ equally spaced angles; heatmaps are smoothed for visualization.
    \method{} (large left panel), baselines, and single-loss ablations share a color scale (lower is better).
    \method{} maintains low error across the evaluated pairs.
  }
  \label{fig:comp_grid_drawer_open}
\end{figure}

\paragraph{Why explicitly regularize composition?}
This asymmetry is consistent with the logical dependence among the underlying axioms. Appendix~\ref{appendix:axiom_redundancy} proves that, when satisfied exactly, composition consistency together with either identity or invertibility implies the remaining axiom. In contrast, identity and invertibility together do not imply composition consistency.

Composition therefore supplies a constraint on how different transformations combine that reconstruction and inverse checks alone cannot ensure. Figure~\ref{fig:comp_grid_drawer_open} illustrates this distinction on Drawer Opening. \method{} maintains low compatibility error across the evaluated angle pairs, while removing $\mathcal L_{\mathrm{comp}}$ introduces regions of elevated error despite retaining identity and invertibility regularization. Removing $\mathcal L_{\mathrm{id}}$ or $\mathcal L_{\mathrm{inv}}$ produces smaller changes across the grid. These results support explicitly regularizing composition. The theoretical result concerns exact constraints, so it does not imply that identity and invertibility penalties are interchangeable during training.

\FloatBarrier
\subsection{Do the discovered symmetries improve robot policy learning?}
\label{sec:application}

Having evaluated symmetry quality, we now test whether the learned transformations benefit policy learning. We use frame-averaged BC~\citep{puny2021frame}: predict actions from transformed observations, map them back to the original action frame, and average them. Experimental method names use the prefix EBC (equivariant behavior cloning), giving EBC-\method{} for our transformations.
The policy averages over $n$ transformations from the cyclic subgroup $H=C_n\subset SO(2)$; construction, training details, and conditions for exact equivariance are given in Appendix~\ref{appendix:equivariant_bc}.

EBC-\method{} improves mean success on all four tasks, raising the average from $80.5\%$ for vanilla BC to $85.1\%$ at $H=C_{16}$ (Table~\ref{tab:ebc_results_main}).
The largest gain is on \textbf{Cube Dropping} ($+9.7$ percentage points).
Frame-averaged BC policies using hand-designed (EBC-ExSym~\citep{wang2023surprising}) or GAN-discovered (EBC-SymGAN and EBC-LaLiGAN) symmetries all have lower average success than vanilla BC.

\begin{table}[htbp]
  \centering
  \caption{Episodic success rates in percent (mean $\pm$ std) at $H=C_{16}$, with vanilla BC as reference. \textbf{Bold} indicates the best per column. Results for smaller subgroups ($C_2, C_4, C_8$) are in Appendix~\ref{appendix:ebc_subgroup_results}.}
  \label{tab:ebc_results_main}
  \begin{adjustbox}{width=1.0\linewidth}
    \begin{tabular}{l|cccc|c}
      \toprule
      Method                   & \textbf{Drawer Opening} & \textbf{Drawer Closing} & \textbf{Cube Dropping}  & \textbf{Hammer Lifting} & \textbf{Average}        \\
      \midrule
      EBC-ExSym ($C_{16}$)     & $64.7 \pm 2.2$          & $86.4 \pm 4.7$          & $80.9 \pm 5.9$          & $85.1 \pm 5.7$          & $79.3 \pm 9.9$          \\
      EBC-SymGAN ($C_{16}$)    & $66.3 \pm 11.3$         & $85.8 \pm 3.6$          & $81.9 \pm 9.4$          & $49.8 \pm 16.3$         & $71.0 \pm 18.1$         \\
      EBC-LaLiGAN ($C_{16}$)   & $65.3 \pm 2.0$          & $73.6 \pm 6.2$          & $84.5 \pm 6.2$          & $61.3 \pm 5.6$          & $71.2 \pm 10.3$         \\
      EBC-\method{} ($C_{16}$) & $\mathbf{71.7 \pm 1.4}$ & $\mathbf{89.8 \pm 2.0}$ & $\mathbf{90.3 \pm 2.9}$ & $\mathbf{88.7 \pm 2.5}$ & $\mathbf{85.1 \pm 8.1}$ \\
      \midrule
      BC                       & $71.4 \pm 0.8$          & $88.1 \pm 5.9$          & $80.6 \pm 2.5$          & $81.8 \pm 3.0$          & $80.5 \pm 7.0$          \\
      \bottomrule
    \end{tabular}
  \end{adjustbox}
\end{table}

\paragraph{Effect of subgroup size.}
Averaging over more transformations steadily improves EBC-\method{}'s mean success rate across the four tasks, whereas EBC-SymGAN and EBC-LaLiGAN decline overall (Figure~\ref{fig:subgroup_size_rgb}).
EBC-ExSym remains below vanilla BC throughout the sweep.
These trends show that the benefit of frame averaging depends on the quality of the underlying symmetries.

\begin{figure}[htbp]
  \centering
  \includegraphics[width=0.6\linewidth]{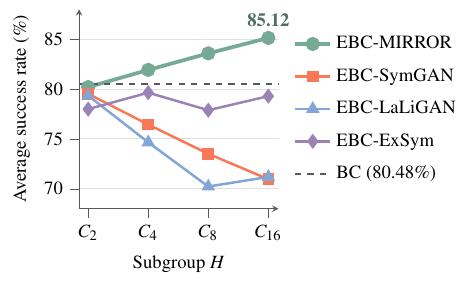}
  \caption{\textbf{RGB policy performance versus subgroup size.}
    Points show mean episodic success rates across four tasks; higher is better.
    For $H=C_n$, the policy uses $n$ symmetry transformations for frame averaging. The dashed line denotes vanilla BC.
    EBC-\method{} improves with subgroup size and exceeds BC at $C_4$, $C_8$, and $C_{16}$; the other frame-averaged variants remain below BC.
    Per-task results and standard deviations are reported in Tables~\ref{tab:ebc_results_main} and~\ref{tab:ebc_results}.}
  \label{fig:subgroup_size_rgb}
\end{figure}

%% file: sections/sim2real.tex
\clearpage
\section{Sim-to-Real Deployment}
\label{sec:sim2real}

To test transfer beyond RGB, we use two-channel observations comprising a depth-derived heightmap and a target-object segmentation mask (Figure~\ref{fig:real_tasks_obs}(e,f)). We rerun symmetry discovery in simulation, freeze the perceptual autoencoder and symmetry modules, and train EBC-\method{} on this modality (Appendix~\ref{appendix:sim2real_success_rate}) before \emph{zero-shot} deployment on a Kinova GEN3 arm. Appendix~\ref{appendix:sim2real_symmetry_vis} visualizes the coupled observation--action transformations.

\begin{figure}[htbp]
  \centering
  \includegraphics[width=0.75\linewidth]{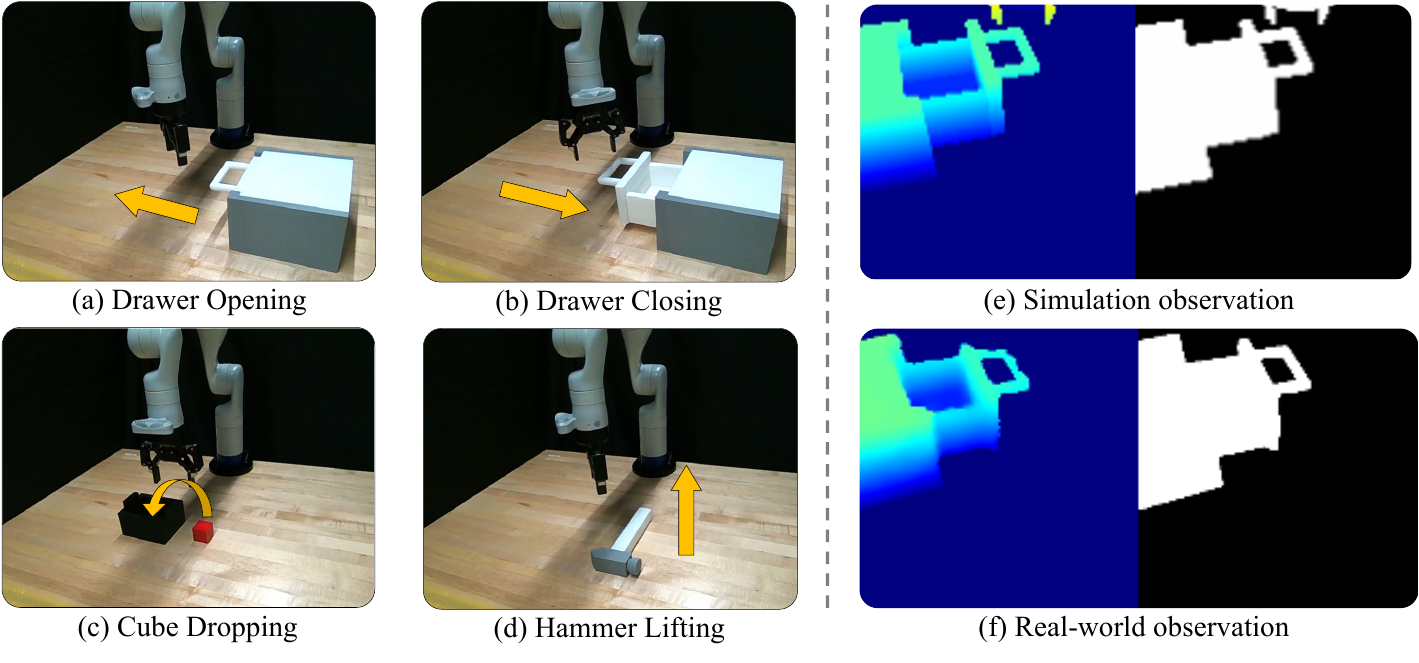}
  \caption{\textbf{Real-world tasks and policy observations.}
    (a)--(d) Arrows indicate intended object motion.
    (e,f) Simulation and real-world observations: heightmap (left) and segmentation mask (right).}
  \label{fig:real_tasks_obs}
\end{figure}

Over 20 trials per task, the policy succeeds on 13/20 Drawer Opening, 19/20 Drawer Closing, 9/20 Cube Dropping, and 16/20 Hammer Lifting, transferring on three of four tasks with Cube Dropping remaining the most challenging. Hardware and preprocessing details are in Appendix~\ref{appendix:real_env_description}.
\FloatBarrier

%% file: sections/conclusion.tex
\section{Conclusion}

We introduced \textbf{\method{}}, a framework for learning how a specified symmetry group operates jointly on observations and actions in pixel-based POMDPs from offline interaction data. We formulate observable symmetry as invariance of the offline one-step transition-tuple distribution. Adversarial tuple matching learns coupled transformations, while identity, compatibility, and invertibility regularization encourages consistency with the group-operation axioms. Across four Robosuite manipulation tasks, \method{} achieves the lowest aggregate error on three of four symmetry-quality metrics, including a $61\%$ reduction in compatibility error relative to the adapted LaLiGAN baseline.

Using these transformations in frame-averaged BC yields higher mean success rates on all four tasks than vanilla BC and the hand-designed and GAN-discovered symmetry baselines at $H=C_{16}$. Gains over vanilla BC reach $9.66$ percentage points on Cube Dropping and $4.64$ percentage points on average. We also repeat symmetry discovery and policy training in simulation with heightmap and segmentation-mask observations, then deploy the resulting policy zero-shot on a Kinova GEN3 arm. Together, these results show that learning coupled observation--action transformations from data can improve visual policy learning without manually specifying how a known symmetry group operates on images and robot commands.

%% file: sections/limitations.tex
\section{Limitations}

\method{} has two scope boundaries worth naming. First, the framework discovers how a group operates on observations and actions but requires the group $G$ itself to be specified in advance; we view this as a mild assumption in practice, since the relevant group is usually evident from task geometry while its observation--action transformations are what is genuinely hard to specify by hand. Second, \method{} can only recover symmetries that are already present in the data: if the dataset lacks variation along the symmetry (for example, a fixed scene orientation across all episodes), the discriminator has nothing to match against, and $(\phi_o, \phi_a)$ may collapse toward the identity.

%% file: appendix/overview.tex
\section*{Appendix Overview}
\begin{itemize}[leftmargin=1.5em]
  \item \hyperref[appendix:theory]{\textbf{Appendix~\ref*{appendix:theory}: Theoretical Analysis}}
  \item \textbf{\hyperref[appendix:env_description]{Appendix~\ref*{appendix:env_description}: Simulation Setup} and \hyperref[sec:eval]{Evaluation Protocol}}
  \item \hyperref[appendix:implementation]{\textbf{Appendix~\ref*{appendix:implementation}: Symmetry-Discovery Implementation}}
  \item \hyperref[appendix:additional_env_results]{\textbf{Appendix~\ref*{appendix:additional_env_results}: Additional Symmetry-Discovery Results}}
  \item \hyperref[appendix:equivariant_bc]{\textbf{Appendix~\ref*{appendix:equivariant_bc}: Frame-Averaged Behavior Cloning}}
  \item \hyperref[appendix:sim2real]{\textbf{Appendix~\ref*{appendix:sim2real}: Sim-to-Real Experiments}}
\end{itemize}

%% file: appendix/prove.tex
\section{Theoretical Analysis}
\label{appendix:theory}

\subsection{Preservation of Group-Operation Axioms under Two-Stage Encoding}
\label{appendix:prove}

We use a two-stage observation encoding scheme: a fixed perceptual autoencoder
$(E_o, D_o)$ maps observations to a compact perceptual latent space,
and the learnable symmetry modules operate on this latent space.
This appendix establishes exact preservation of the group-operation axioms when the latent operation satisfies them and the perceptual encoder and decoder are mutual inverses on the relevant invariant support.

\paragraph{Setup.}
Let $G$ be a group with identity element $e$ and composition $\circ$.
Let $E_o:\mathcal O\to\mathcal Z$ and $D_o:\mathcal Z\to\mathcal O$
be the fixed perceptual encoder/decoder, and let $\phi_z:G\times\mathcal Z\to\mathcal Z$
be the learned transformation on perceptual latents.
We define the induced observation-space transformation by
\begin{equation}
  \phi_o(g,o)\ :=\ D_o\!\big(\phi_z(g, E_o(o))\big).
  \label{eq:phi_o_two_stage_app}
\end{equation}

\begin{samepage}
  \begin{lemma}[Conjugation preserves group operations]
    \label{lem:conjugation_group_operation}
    Let $\mathcal O_\ast\subseteq\mathcal O$ be a relevant observation support and define $\mathcal Z_\ast=E_o(\mathcal O_\ast)$.
    Assume that $\mathcal Z_\ast$ is $G$-invariant under $\phi_z$, meaning $\phi_z(g,z)\in\mathcal Z_\ast$ for every $g\in G$ and $z\in\mathcal Z_\ast$.
    Assume that $\phi_z$ is a (left) group operation on $\mathcal Z_\ast$, i.e., for every $z\in\mathcal Z_\ast$,
    \begin{align}
      \phi_z(e,z)               & = z, \label{eq:group_axiom_id_z}                                                     \\
      \phi_z(g_2,\phi_z(g_1,z)) & = \phi_z(g_2\circ g_1, z), \quad \forall g_1,g_2\in G. \label{eq:group_axiom_comp_z}
    \end{align}
    Further assume that $(E_o,D_o)$ are mutual inverses on these supports, meaning
    \begin{equation}
      D_o(E_o(o)) = o\quad(o\in\mathcal O_\ast),
      \qquad
      E_o(D_o(z)) = z\quad(z\in\mathcal Z_\ast).
      \label{eq:ae_inverse_assumption}
    \end{equation}
    Then the induced map $\phi_o$ in \eqref{eq:phi_o_two_stage_app} is a group operation on $\mathcal O_\ast$: for every $o\in\mathcal O_\ast$,
    \begin{align}
      \phi_o(e,o)               & = o, \label{eq:group_axiom_id_o}                                                     \\
      \phi_o(g_2,\phi_o(g_1,o)) & = \phi_o(g_2\circ g_1, o), \quad \forall g_1,g_2\in G. \label{eq:group_axiom_comp_o}
    \end{align}
  \end{lemma}
\end{samepage}

\begin{proof}
  Fix $o\in\mathcal O_\ast$. Since $D_o(\mathcal Z_\ast)=\mathcal O_\ast$ and $\mathcal Z_\ast$ is $G$-invariant, $\phi_o(g,o)\in\mathcal O_\ast$ for every $g\in G$, so the induced operation is closed on the stated support.

  \emph{Identity.} Using \eqref{eq:phi_o_two_stage_app}, \eqref{eq:group_axiom_id_z}, and \eqref{eq:ae_inverse_assumption},
  \[
    \phi_o(e,o)
    = D_o\!\big(\phi_z(e,E_o(o))\big)
    = D_o\!\big(E_o(o)\big)
    = o.
  \]

  \emph{Compatibility.} Using \eqref{eq:phi_o_two_stage_app} twice, then \eqref{eq:ae_inverse_assumption} to cancel $E_o\circ D_o$,
  and finally \eqref{eq:group_axiom_comp_z},
  \begin{align*}
    \phi_o(g_2,\phi_o(g_1,o))
     & = D_o\!\Big(\phi_z\big(g_2,\ E_o(\phi_o(g_1,o))\big)\Big)           \\
     & = D_o\!\Big(\phi_z\big(g_2,\ E_o(D_o(\phi_z(g_1,E_o(o))))\big)\Big) \\
     & = D_o\!\Big(\phi_z\big(g_2,\ \phi_z(g_1,E_o(o))\big)\Big)           \\
     & = D_o\!\Big(\phi_z\big(g_2\circ g_1,\ E_o(o)\big)\Big)              \\
     & = \phi_o(g_2\circ g_1,o).
  \end{align*}
  This proves \eqref{eq:group_axiom_id_o}--\eqref{eq:group_axiom_comp_o}.
  Invertibility follows by taking $g_2=g_1^{-1}$ in the compatibility identity and applying the identity axiom.
\end{proof}

\paragraph{Interpretation.}
Under the lemma's assumptions, the perceptual autoencoder is a change of coordinates between $\mathcal O_\ast$ and $\mathcal Z_\ast$.
An exact latent group operation therefore induces an exact observation-space group operation on the corresponding support.
The conclusion depends on the stated inverse and invariance conditions; pretraining and freezing the autoencoder alone does not guarantee them.

\paragraph{Remark (learned approximations).}
The lemma establishes exact preservation under its stated assumptions.
In practice, the perceptual autoencoder is approximately invertible, and the learned latent operation is encouraged to satisfy the group-operation axioms through regularization.
Only the fixed linear representations $\rho_z$ and $\rho_a$ satisfy the representation identities exactly by construction.
Consequently, the learned observation-space operation may exhibit residual axiom violations, which we quantify using the evaluation metrics in Appendix~\ref{sec:eval}.

\subsection{Logical Dependence among the Group-Operation Axioms}
\label{appendix:axiom_redundancy}

The \method{} objective uses three group-consistency regularizers: identity ($\mathcal{L}_{\mathrm{id}}$), compatibility ($\mathcal{L}_{\mathrm{comp}}$), and invertibility ($\mathcal{L}_{\mathrm{inv}}$) in Eq.~\eqref{eq:axiom_losses}.
This section analyzes the logical dependence among the corresponding exact constraints: given compatibility, identity and invertibility imply one another, whereas identity and invertibility together do not imply compatibility.

\paragraph{Setup.}
Let $G$ be a group with identity element $e$ and composition $\circ$, and let $\phi:G\times\mathcal{X}\to\mathcal{X}$ denote a candidate group operation on a set $\mathcal{X}$.
The analysis applies to both $\Phi$ and $\Phi_z$ on their respective tuple spaces. We consider the following axioms, each required to hold exactly for every $g,g_1,g_2\in G$ and $x\in\mathcal{X}$:
\begin{align}
  \text{(Identity)} \qquad      & \phi(e,x) = x, \label{eq:axiom_id}                                  \\
  \text{(Compatibility)} \qquad & \phi(g_2,\phi(g_1,x)) = \phi(g_2\circ g_1,x), \label{eq:axiom_comp} \\
  \text{(Invertibility)} \qquad & \phi(g^{-1},\phi(g,x)) = x. \label{eq:axiom_inv}
\end{align}

\begin{lemma}[Identity from compatibility and invertibility]
  \label{lem:id_from_comp_inv}
  If $\phi$ satisfies compatibility \eqref{eq:axiom_comp} and invertibility \eqref{eq:axiom_inv}, then $\phi$ satisfies identity \eqref{eq:axiom_id}.
\end{lemma}

\begin{proof}
  Fix any $g\in G$ and any $x\in\mathcal{X}$. Since $e=g^{-1}\circ g$,
  \[
    \phi(e,x)
    \;=\;\phi(g^{-1}\circ g,x)
    \;\stackrel{\eqref{eq:axiom_comp}}{=}\;\phi(g^{-1},\phi(g,x))
    \;\stackrel{\eqref{eq:axiom_inv}}{=}\;x,
  \]
  which is the identity axiom.
\end{proof}

\begin{lemma}[Invertibility from compatibility and identity]
  \label{lem:inv_from_comp_id}
  If $\phi$ satisfies compatibility \eqref{eq:axiom_comp} and identity \eqref{eq:axiom_id}, then $\phi$ satisfies invertibility \eqref{eq:axiom_inv}.
\end{lemma}

\begin{proof}
  For any $g\in G$ and $x\in\mathcal{X}$,
  \[
    \phi(g^{-1},\phi(g,x))
    \;\stackrel{\eqref{eq:axiom_comp}}{=}\;\phi(g^{-1}\circ g,x)
    \;=\;\phi(e,x)
    \;\stackrel{\eqref{eq:axiom_id}}{=}\;x,
  \]
  which is the invertibility axiom.
\end{proof}

\paragraph{Independence of compatibility.}
Identity and invertibility together do not imply compatibility.
To see this, take $G=SO(2)$ and $\mathcal{X}=\mathbb{R}$, and define $\phi(R_\theta,x)=x+\sin\theta$, where $R_\theta$ denotes rotation by angle $\theta$.
This map is well-defined on $SO(2)$ because $\sin\theta$ is $2\pi$-periodic, and it satisfies identity and invertibility for every $x\in\mathbb{R}$ and $\theta$:
\[
  \phi(R_0,x)=x,\qquad
  \phi(R_{-\theta},\phi(R_\theta,x))=x+\sin\theta-\sin\theta=x.
\]
However, compatibility fails, since
\[
  \phi(R_{\pi/2},\phi(R_{\pi/2},x))=x+2\neq x=\phi(R_\pi,x).
\]

\paragraph{Consequence for the exact constraints.}
Lemmas~\ref{lem:id_from_comp_inv} and~\ref{lem:inv_from_comp_id} establish that, under exact compatibility, identity and invertibility imply one another.
Consequently, either identity or invertibility is redundant in an exact characterization that includes the other two axioms, while compatibility cannot be inferred from identity and invertibility alone.
This logical dependence concerns the axiom constraints themselves and does not imply that their corresponding optimization penalties are interchangeable.

%% file: appendix/env_description.tex
\FloatBarrier
\section{Simulation Setup and Evaluation Protocol}
\label{appendix:env_description}

We first specify the simulation setup, observation and action coordinates, and per-task reward functions. We then describe offline data collection in Section~\ref{appendix:behavior_policy}, followed by the evaluation metrics and group sampling in Section~\ref{sec:eval}.

\subsection{Robosuite Setup}
\label{appendix:robosuite_setup}

All four tabletop tasks share the same Robosuite hardware and perception configuration, chosen to mirror our physical platform so that the same observation and action interfaces can be reused in the sim-to-real study of Appendix~\ref{appendix:real_env_description}.
The simulated arm is a $7$-DoF Kinova GEN3 manipulator equipped with a Robotiq 2F85 parallel-jaw gripper, matching the real robot used for zero-shot transfer.
Rather than adopting one of Robosuite's built-in camera presets, we calibrate the simulator camera against the real workspace: we first mount an Intel RealSense D435 in front of the table and run a standard eye-to-hand calibration to recover its intrinsic matrix and its extrinsic pose with respect to the robot base, and we then transcribe both matrices into the simulator so that the virtual camera reproduces the same viewpoint, field of view, and projection geometry as the physical sensor.
This calibration-driven setup yields RGB observations whose camera geometry is aligned between simulation and the real robot, eliminating viewpoint-induced sim-to-real gap before any policy transfer.

\subsection{Per-Task Observation and Action Coordinates}
\label{appendix:obs_action_coords}

The observation space $\mathcal{O}$ is the same $3$-channel RGB image from a fixed camera (Figure~\ref{fig:env_figure_view}) for all four tasks; the action space $\mathcal{A}$ consists of task-dependent, normalized end-effector commands. Throughout the experiments and action visualizations, $(\Delta x,\Delta y,\Delta z)$, $\Delta\psi$, and $\Delta w$ denote dimensionless, normalized commands for Cartesian translation, yaw, and gripper opening, respectively.
For Drawer Opening and Drawer Closing, actions are $4$-dimensional, $(\Delta x,\Delta y,\Delta z,\Delta\psi)$.
For Cube Dropping, actions are also $4$-dimensional, $(\Delta x,\Delta y,\Delta z,\Delta w)$.
For Hammer Lifting, actions are $5$-dimensional, $(\Delta x,\Delta y,\Delta z,\Delta\psi,\Delta w)$.
Each action component uses zero-centered scaling, $a_i=\delta_i/s_i$, where $\delta_i$ is the physical command increment and $s_i>0$ is its per-coordinate scale. The normalized commands lie in $[-1,1]$, so $\mathcal A\subseteq[-1,1]^{d_a}$ with the task-specific dimension $d_a$ given above. The corresponding per-task vector-valued reward functions used by the discriminator are defined in Section~\ref{appendix:reward_functions} below.

\subsection{Environment Reward Functions}
\label{appendix:reward_functions}

The discriminator receives a vector-valued reward whose ordered components are listed in Table~\ref{tab:reward_composition}. Cube Dropping uses six components; the other tasks use four. We define the shared components once below and then specify the task-specific terms.

\paragraph{Shared reaching and action terms.}
Let $\mathbf{p}^{\mathrm{grp}}_t\in\mathbb R^3$ be the gripper position and $\mathbf{p}^{\mathrm{target}}_t\in\mathbb R^3$ the position of the task's reaching target (Table~\ref{tab:reward_composition}). All four tasks use
\begin{equation}
  \begin{aligned}
    d^{\mathrm{reach}}_t & := \left\lVert\mathbf{p}^{\mathrm{grp}}_t-\mathbf{p}^{\mathrm{target}}_t\right\rVert_2, \\
    r^{\mathrm{reach}}_t & := 1-\tanh\!\left(10\,d^{\mathrm{reach}}_t\right),
    \qquad r^{\mathrm{action}}_t := -\left\lVert\mathbf a_t\right\rVert_2.
  \end{aligned}
\end{equation}

\paragraph{Drawer progress and alignment.}
For both drawer tasks, let $p_t\in[0,1]$ denote the drawer-out fraction and $\epsilon_t^{\mathrm{yaw}}$ the yaw difference, in radians, between the gripper and the handle target. The task-specific terms are
\begin{equation}
  r^{\mathrm{open}}_t := \tanh(p_t),\qquad
  r^{\mathrm{close}}_t := 1-\tanh(p_t),\qquad
  r^{\mathrm{yaw}}_t := \cos^2(\epsilon_t^{\mathrm{yaw}}).
\end{equation}

\paragraph{Grasping and lifting.}
For Cube Dropping and Hammer Lifting, let $g_t\in\{0,1\}$ indicate whether the respective object is grasped and $\Delta h_t$ denote its height above the table surface. Both tasks use
\begin{equation}
  r^{\mathrm{grasp}}_t := g_t,\qquad
  r^{\mathrm{lift}}_t := \tanh\!\left(10\,\Delta h_t\right).
\end{equation}

\paragraph{Cube alignment and success.}
Let $\mathbf p^{\mathrm{cube}}_t,\mathbf p^{\mathrm{pot}}_t\in\mathbb R^3$ be the cube and pot positions. Alignment above the pot is measured in the horizontal plane:
\begin{equation}
  d^{\mathrm{align}}_t := \left\lVert\mathbf p^{\mathrm{cube}}_{t,1:2}-\mathbf p^{\mathrm{pot}}_{t,1:2}\right\rVert_2,
  \qquad r^{\mathrm{align}}_t := 1-\tanh\!\left(10\,d^{\mathrm{align}}_t\right).
\end{equation}
The success reward $r^{\mathrm{success}}_t=1$ when the cube is in the pot and the gripper has released it.

\begin{table}[htbp]
  \centering
  \caption{Per-task reward composition. Components are listed in their order in the reward vector; the reaching target specifies $\mathbf p^{\mathrm{target}}_t$.}
  \label{tab:reward_composition}
  \small
  \setlength{\tabcolsep}{4pt}
  \renewcommand{\arraystretch}{1.15}
  \begin{tabular}{l|ll}
    \toprule
    Task           & Reaching target & Reward vector $\mathbf r_t^{\top}$                                                                                                  \\
    \midrule
    Drawer Opening & Drawer handle   & $(r^{\mathrm{reach}}_t,r^{\mathrm{open}}_t,r^{\mathrm{yaw}}_t,r^{\mathrm{action}}_t)$                                               \\
    Drawer Closing & Drawer handle   & $(r^{\mathrm{reach}}_t,r^{\mathrm{close}}_t,r^{\mathrm{yaw}}_t,r^{\mathrm{action}}_t)$                                              \\
    Cube Dropping  & Cube            & $(r^{\mathrm{reach}}_t,r^{\mathrm{grasp}}_t,r^{\mathrm{lift}}_t,r^{\mathrm{align}}_t,r^{\mathrm{success}}_t,r^{\mathrm{action}}_t)$ \\
    Hammer Lifting & Hammer          & $(r^{\mathrm{reach}}_t,r^{\mathrm{grasp}}_t,r^{\mathrm{lift}}_t,r^{\mathrm{action}}_t)$                                             \\
    \bottomrule
  \end{tabular}
\end{table}

\subsection{Offline Dataset and Behavior Policy}
\label{appendix:behavior_policy}

For each task, we collect $100{,}000$ one-step transition tuples $(o,a,o',r)$ for training and $10{,}000$ for evaluation, as described in Section~\ref{sec:exp_setup}. Each tuple contains the current and next RGB observations, the normalized end-effector command, and the task-specific reward vector defined above.

The behavior policy used to collect the offline interaction data for all four tasks is a simple finite-state-machine controller that drives the robot end-effector through a sequence of task-specific waypoints using proportional pose control. The waypoint sequence for each task is listed below.

\paragraph{Drawer Opening.}
\begin{enumerate}
  \item a pre-grasp pose about $5$\,cm above the drawer handle;
  \item a grasp pose about $2$\,cm below the handle to establish contact;
  \item a pulling pose about $10$\,cm along the drawer-opening direction.
\end{enumerate}

\paragraph{Drawer Closing.}
\begin{enumerate}
  \item an approach pose about $10$\,cm in front of the handle and $2$\,cm above it;
  \item a contact pose at the same front offset but $2$\,cm below the handle height;
  \item a pushing pose about $10$\,cm past the handle in the closing direction.
\end{enumerate}

\paragraph{Cube Dropping.}
\begin{enumerate}
  \item an approach pose about $20$\,cm above the cube with the gripper open;
  \item a descend pose about $2$\,cm below the cube reference position for grasping;
  \item a lift pose about $20$\,cm above the cube with the gripper closed;
  \item a transfer pose about $20$\,cm above the pot;
  \item a drop pose at the same hover position above the pot, where the gripper opens to release the cube.
\end{enumerate}

\paragraph{Hammer Lifting.}
\begin{enumerate}
  \item an approach pose about $20$\,cm above the hammer with the gripper open;
  \item a descend pose about $2$\,cm below the hammer reference position to align for grasping;
  \item a grasp pose at the same position with the gripper closed;
  \item a lift pose about $20$\,cm upward.
\end{enumerate}

\input{appendix/evaluation_metrics}

%% file: appendix/evaluation_metrics.tex
\subsection{Evaluation Metrics and Protocol}
\label{sec:eval}

We evaluate the learned observation--action transformations $(\phi_o,\phi_a)$ along two criteria:
(1) \emph{observable symmetry}, i.e., whether transforming tuples preserves the offline transition-tuple distribution;
and (2) \emph{group-operation validity}, i.e., whether the learned maps satisfy basic group-operation axioms
(identity, compatibility, and invertibility).

\subsubsection{Observable Symmetry Evaluation}
We measure observable symmetry by the \textbf{Symmetry invariance error},
the Wasserstein-2 distance between the original dataset $\mathcal D$ and its
transformed version $\Phi(g,\mathcal D)$:
\begin{equation}
  \mathcal{E}_{\mathrm{sym}}(g) \;=\; W_2\!\left(\mathcal{D},\Phi(g,\mathcal{D})\right).
  \label{eq:wasserstein_main}
\end{equation}
A perfect symmetry yields $\mathcal{E}_{\mathrm{sym}}(g)=0$ (i.e., $p(\tau)=p(\Phi(g,\tau))$),
and smaller values mean $\Phi(g,\mathcal{D})$ remains closer in distribution to $\mathcal D$.

\paragraph{Ground cost in feature space.}
Direct optimal transport in raw pixel space is typically uninformative and expensive.
Instead, we embed observations using a fixed pretrained visual encoder $f_o:\mathcal O\to\mathbb R^{d_o}$ (e.g., ResNet-18 features) and compare tuples $\tau=(o,a,o',r)$ and $\tau'=(\bar o,\bar a,\bar o',\bar r)$ via the per-tuple feature distance
\begin{equation}
  d_f(\tau,\tau') \;:=\; \|f_o(o)-f_o(\bar o)\|_2^2 + \|a-\bar a\|_2^2 + \|f_o(o')-f_o(\bar o')\|_2^2,
  \label{eq:tuple_feature_dist}
\end{equation}
adding $\|r-\bar r\|_2^2$ as the Wasserstein ground cost (since arbitrary tuples may genuinely differ in reward) and using $d_f$ directly for the validity errors below (since $\Phi$ preserves $r$ by construction).

\subsubsection{Group-Operation Validity Evaluation}
Beyond distribution matching, $(\phi_o,\phi_a)$ should also satisfy the three group-operation axioms when applied to tuples: identity should leave a tuple unchanged, applying two transformations in sequence should equal applying their composition, and a transformation followed by its inverse should return to the original. For fixed group elements $g,g_1,g_2\in G$, we define three feature-space errors
\begin{equation}
  \begin{aligned}
    \mathcal{E}_{\mathrm{id}}            & = \mathbb{E}_{\tau\sim\mathcal{D}}\big[\, d_f(\Phi(e,\tau),\,\tau)\, \big],                                \\
    \mathcal{E}_{\mathrm{comp}}(g_1,g_2) & = \mathbb{E}_{\tau\sim\mathcal{D}}\big[\, d_f(\Phi(g_2,\Phi(g_1,\tau)),\,\Phi(g_2\circ g_1,\tau))\, \big], \\
    \mathcal{E}_{\mathrm{inv}}(g)        & = \mathbb{E}_{\tau\sim\mathcal{D}}\big[\, d_f(\Phi(g^{-1},\Phi(g,\tau)),\,\tau)\, \big].
  \end{aligned}
  \label{eq:axiom_errors}
\end{equation}

\subsubsection{Group Sampling for Evaluation}
\label{appendix:eval_sampling}
At evaluation time, $\overline{\mathcal E}_{\mathrm{sym}}$ and $\overline{\mathcal E}_{\mathrm{inv}}$ denote averages of $\mathcal E_{\mathrm{sym}}(g)$ and $\mathcal E_{\mathrm{inv}}(g)$ over $256$ angles uniformly sampled from $SO(2)$; identity error $\mathcal E_{\mathrm{id}}$ is evaluated only at $g=e$. For $\overline{\mathcal E}_{\mathrm{comp}}$, we average $\mathcal E_{\mathrm{comp}}(g_1,g_2)$ over all $16\times16=256$ ordered pairs of equally spaced rotation angles spanning the full circle. The compatibility heatmaps in Appendix~\ref{appendix:comp_angle_all_envs} display per-pair errors on this same grid. Qualitative visualizations of the learned transformations and of latent-space distribution matching use the discrete subgroup $C_4=\{e,h,h^2,h^3\}\subset SO(2)$ with $\theta_h=90^\circ$ in the linear latent representations.

%% file: appendix/implementation_details.tex
\FloatBarrier
\section{Symmetry-Discovery Implementation}
\label{appendix:implementation}

This appendix specifies the network architectures (Section~\ref{appendix:arch}), fixed group representations (Section~\ref{appendix:representations}), training objectives and procedure (Section~\ref{appendix:discovery_procedure}), and hyperparameters (Section~\ref{appendix:training_discovery}) for symmetry discovery. Evaluation metrics and evaluation-time group sampling are described in Appendix~\ref{sec:eval}, and the frame-averaged BC policy is described in Appendix~\ref{appendix:equivariant_bc}.

\subsection{Network Architectures}
\label{appendix:arch}

\paragraph{Perceptual autoencoder $(E_o, D_o)$ (Stage~1).}
We instantiate the perceptual autoencoder as a convolutional spherical autoencoder~\citep{zhao2019latent} that maps each workspace observation to a unit-norm perceptual latent and back, providing the smooth perceptual latent space on which Stage~2 symmetry discovery operates.
The network accepts a channel-first $C\times 128\times 128$ tensor with $C=3$ for the raw RGB images used in the simulation-only experiments and $C=2$ for the stacked heightmap and segmentation mask used in the sim-to-real pipeline.
The encoder is a five-stage convolutional tower with channel widths $4,8,16,32,64$, SiLU activations, and layer normalization, interleaved with four average-pooling downsamples that produce a $64\times 8\times 8$ bottleneck.
A linear projection followed by $\ell_2$ normalization then sends this bottleneck to a perceptual latent $z=E_o(o)\in\mathbb S^{15}$ with ambient dimension $d_z=16$.
The decoder mirrors this structure: a linear expansion lifts $z$ back to a $64\times 8\times 8$ feature map, bilinear-upsampling blocks with channel widths $128,64,32,16,8$ progressively restore spatial resolution, and a final $3\times 3$ convolution followed by a $\tanh$ activation and per-channel range de-normalization recovers the input tensor.

\paragraph{Symmetry modules $(E_z, D_z)$ and $(E_a, D_a)$ (Stage~2).}
The Stage~2 symmetry modules are two residual MLP autoencoders that wrap the Stage~1 perceptual latent and the environment action in their own learnable spherical symmetry latent spaces.
The observation symmetry encoder $E_z$ maps the Stage~1 perceptual latent $z\in\mathbb S^{15}$ to an observation symmetry latent $u=E_z(z)$ with dimension $d_u=64$, and the observation symmetry decoder $D_z$ maps it back to $\mathbb S^{15}$.
The action symmetry encoder $E_a$ maps each normalized environment action $a\in\mathcal A\subseteq[-1,1]^{d_a}$ to an action symmetry latent $v=E_a(a)$ with dimension $d_v=16$, and the action symmetry decoder $D_a$ maps it back to normalized action coordinates.

All four modules share a common MLP template.
Each encoder consists of a linear projection $\mathrm{Linear}(d_{\mathrm{in}},128)$, a single width-$128$ residual block, and a final $\mathrm{GELU}$--$\mathrm{Linear}(128,d_{\mathrm{lat}})$ projection.
Each decoder is wider: a linear expansion $\mathrm{Linear}(d_{\mathrm{lat}},256)$, a single width-$256$ residual block, and a final $\mathrm{GELU}$--$\mathrm{Linear}(256,d_{\mathrm{out}})$ projection.
The observation branch uses $d_{\mathrm{in}}=d_{\mathrm{out}}=d_z$ and $d_{\mathrm{lat}}=d_u$; the action branch uses $d_{\mathrm{in}}=d_{\mathrm{out}}=d_a$ and $d_{\mathrm{lat}}=d_v$.
Every residual block applies $\mathrm{LayerNorm}$--$\mathrm{GELU}$--$\mathrm{Linear}$--$\mathrm{LayerNorm}$--$\mathrm{GELU}$--$\mathrm{Linear}$ with an additive skip connection, and we use neither spectral normalization nor orthogonal parametrization.

\paragraph{Normalization and initialization.}
All Stage~2 symmetry latents are projected onto the unit hypersphere, and the output of $D_z$ is additionally projected back onto the Stage~1 perceptual latent space $\mathbb S^{15}$. Keeping transformed perceptual latents on the same sphere as the original latents facilitates adversarial tuple matching.
Linear layers use Kaiming initialization tuned for GELU activations, except for the final encoder projection, which uses Xavier-normal initialization with gain $0.01$ and zero bias.
The decomposition of each symmetry latent into rotating and invariant coordinates is specified in Section~\ref{appendix:representations}.

\paragraph{Discriminator $f_\psi$.}
The tuple discriminator $f_\psi$ is a spectral-normalized MLP with two scoring heads: a transition head $f_\psi^{\mathrm{mdp}}$ that scores latent transition tuples jointly and an auxiliary marginal head $f_\psi^{\mathrm{obs}}$ that scores perceptual latents on their own.
The transition head ingests a latent transition tuple $(z,a,z',r)$, where $z,z'\in\mathbb S^{15}$ are Stage~1 perceptual latents, $a\in\mathbb R^{d_a}$ is the environment action, and $r$ is the per-task vector-valued reward defined in Appendix~\ref{appendix:reward_functions}.
Each of the four components is first embedded into $\mathbb R^{32}$ by a dedicated spectral-normalized linear subnet, and the four embeddings are concatenated to form a $128$-dimensional tuple feature.
This feature is then mapped to a scalar logit via three spectral-normalized linear layers, interleaved with leaky ReLU activations.
The marginal head reuses the same backbone but starts from a single $\mathbb R^{128}$ embedding of a perceptual latent; we apply it independently to $z$ and $z'$ and concatenate the two scalar scores to form its output.

\subsection{Fixed Group Representations}
\label{appendix:representations}

We instantiate $\rho_z$ and $\rho_a$ as fixed block-diagonal representations of $g\in SO(2)$ in the learned observation and action symmetry latents.
The angle $\theta_g\in[0,2\pi)$ parameterizes $g$, and $R(\theta_g)\in\mathbb R^{2\times 2}$ is the standard $2$D rotation matrix.
For either input branch, indexed by $\bullet\in\{z,a\}$, the representation rotates $K_\bullet$ two-dimensional blocks and preserves $M_\bullet$ invariant coordinates:
\begin{equation*}
  \rho_\bullet(g)=\operatorname{diag}\!\big(\underbrace{R(\theta_g),\ldots,R(\theta_g)}_{K_\bullet\text{ copies}},\,I_{M_\bullet}\big).
\end{equation*}
Thus, $d_u=2K_z+M_z$ and $d_v=2K_a+M_a$, with the dimensions listed in Table~\ref{tab:discovery_latent_dimensions}.

\begin{table}[!htbp]
  \centering
  \caption{Decomposition of the Stage~2 symmetry latents. Each rotating block has two coordinates.}
  \label{tab:discovery_latent_dimensions}
  \small
  \begin{tabular}{l|ccc}
    \toprule
    Symmetry latent & \shortstack{Rotating blocks                   \\$K_\bullet$} & \shortstack{Invariant coordinates\\$M_\bullet$} & Total dimension \\
    \midrule
    Observation $u$ & $16$                        & $32$ & $d_u=64$ \\
    Action $v$      & $4$                         & $8$  & $d_v=16$ \\
    \bottomrule
  \end{tabular}
\end{table}

Decoding the transformed symmetry latents gives the operations
\begin{equation*}
  \phi_z(g,z)=D_z\!\big(\rho_z(g)E_z(z)\big),
  \qquad
  \phi_a(g,a)=D_a\!\big(\rho_a(g)E_a(a)\big).
\end{equation*}
Here, $\phi_z$ acts on the Stage~1 perceptual latent, while $\phi_a$ returns actions in the normalized coordinates defined in Appendix~\ref{appendix:obs_action_coords}.
The corresponding observation-space operation is $\phi_o(g,o)=D_o(\phi_z(g,E_o(o)))$.
As in Section~\ref{sec:param}, the angle specifies a rotation in the learned symmetry latents; the induced nonlinear maps $\phi_o$ and $\phi_a$ need not be geometric rotations by the same angle.

\FloatBarrier
\subsection{Training Objectives and Procedure}
\label{appendix:discovery_procedure}

\paragraph{Stage~1: Perceptual pretraining.}
We train $(E_o,D_o)$ on observations from $\mathcal D$ with a normalized $L_1$ reconstruction loss after an initial $L_2$ warm-up.
A latent cycle-consistency term encourages $E_o(D_o(z))\approx z$.
When enabled, a perceptual discriminator supplies adversarial and feature-matching losses to sharpen reconstruction texture; their weights are introduced with a linear warm-up.
The warm-up durations and optimization settings are collected in Section~\ref{appendix:training_discovery}.

\paragraph{Transition to Stage~2.}
After pretraining, we freeze $(E_o,D_o)$ and encode each transition tuple $\tau=(o,a,o',r)$ from the offline dataset as a latent transition tuple $t=(z,a,z',r)$ with $z=E_o(o)$ and $z'=E_o(o')$.
Stage~2 trains the symmetry modules $(E_z,D_z,E_a,D_a)$ and the tuple discriminator $f_\psi$.
For a sampled $g$, the transformed tuple $\Phi_z(g,t)$ applies $\phi_z$ to both perceptual latents and $\phi_a$ to the action, leaving the reward unchanged.

\paragraph{Stage~2: Adversarial tuple matching and group-consistency regularization.}
We implement the adversarial tuple-matching objective of Eq.~\eqref{eq:adv_critic} as a multi-head hinge-GAN on raw logits, applied to each scoring head as an independent discriminator and summed in a multi-discriminator fashion.
Concretely, the transition head contributes the hinge discriminator loss
\begin{equation*}
  \mathcal L_{\mathrm{disc}}^{\mathrm{mdp}}
  \;=\; \tfrac{1}{2}\,\mathbb E_{t}\!\left[\max(0,\,1-f_\psi^{\mathrm{mdp}}(t))\right]
  \;+\; \tfrac{1}{2}\,\mathbb E_{t,g}\!\left[\max(0,\,1+f_\psi^{\mathrm{mdp}}(\Phi_z(g,t)))\right],
\end{equation*}
and the marginal head contributes the same hinge form applied independently to each perceptual latent in the tuple, treating $z$ and $\phi_z(g,z)$ as a real/fake pair and likewise for $z'$ and $\phi_z(g,z')$.
The discriminator minimizes the total discriminator loss $\mathcal L_{\mathrm{disc}} = \mathcal L_{\mathrm{disc}}^{\mathrm{mdp}} + \mathcal L_{\mathrm{disc}}^{\mathrm{obs}}$. The symmetry modules act as the generator and minimize $\mathcal L_{\mathrm{gen}} = -\,\mathbb E_{t,g}[f_\psi(\Phi_z(g,t))]$, summed across heads in the same multi-head fashion. This loss is the hinge realization of the $\mathcal L_{\mathrm{adv}}$ term that enters the symmetry modules' total loss in Eq.~\eqref{eq:full_obj}.
The symmetry modules additionally minimize the identity, compatibility, and invertibility regularizers of Eq.~\eqref{eq:axiom_losses}, and the nontriviality penalty $\mathcal L_{\mathrm{nt}}$ between $z$ and $\phi_z(g,z)$ and between $z'$ and $\phi_z(g,z')$. Their weighted sum with the adversarial tuple-matching loss is specified in Eq.~\eqref{eq:full_obj}.

\paragraph{Input normalization and alternating updates.}
Perceptual latents and actions are passed to $f_\psi$ without further normalization, since the latents already lie on $\mathbb S^{15}$ by construction and the environment actions are already normalized to $[-1,1]^{d_a}$; reward components are standardized to zero mean and unit variance across $\mathcal D$.
Stabilization relies on the discriminator's spectral normalization together with gradient clipping; we use neither a gradient penalty nor an $R_1$ regularizer.
The discriminator and the symmetry modules are updated one-to-one: one discriminator step follows every symmetry-module step.

\paragraph{Training-time group sampling.}
We draw rotation angles as $\theta_g\sim\mathcal N(0,\pi^2)$, interpreted modulo $2\pi$.
This sampler has support over the full circle and includes no reflections, consistent with $G=SO(2)$.
Evaluation-time group sampling is specified separately in Appendix~\ref{appendix:eval_sampling}.

\subsection{Training Hyperparameters}
\label{appendix:training_discovery}

Table~\ref{tab:discovery_optimization} collects the core optimization settings for both stages.
The Stage~1 autoencoder learning rate is held constant throughout training, and both the Stage~2 symmetry modules and tuple discriminator use Adam betas $(0.5,0.999)$.

\begin{table}[!htbp]
  \centering
  \caption{Core symmetry-discovery optimization settings. Stage~2 settings apply to both the symmetry modules and tuple discriminator, except that the listed batch size is for symmetry-module updates.}
  \label{tab:discovery_optimization}
  \small
  \begin{tabular}{l|cc}
    \toprule
    Setting                        & \shortstack{Stage~1                     \\Perceptual autoencoder} & \shortstack{Stage~2\\Symmetry discovery} \\
    \midrule
    Optimizer                      & Adam                & Adam              \\
    Learning rate                  & $2\times 10^{-4}$   & $1\times 10^{-4}$ \\
    Gradient updates               & $100{,}000$         & $100{,}000$       \\
    Batch size                     & $32$                & $4096$            \\
    Weight decay                   & $0$                 & $0$               \\
    Gradient clipping (max.\ norm) & $1.0$               & $1.0$             \\
    \bottomrule
  \end{tabular}
\end{table}

\paragraph{Stage~1 warm-up and auxiliary losses.}
The reconstruction loss uses $L_2$ for the first $2{,}000$ iterations before switching to normalized $L_1$.
When the perceptual discriminator and feature-matching loss are enabled, the discriminator uses Adam with betas $(0.5,0.999)$, batch size $32$, and gradient clipping at norm $1.0$.
Its learning rate is $2\times 10^{-4}$ for RGB inputs and $1\times 10^{-4}$ for heightmap and segmentation-mask inputs.
The adversarial and feature-matching loss weights increase linearly from zero to $0.1$ and $2.0$, respectively, over the first $2{,}000$ iterations.

\paragraph{Stage~2 loss weights.}
The Stage~2 objective $\mathcal L_{\mathrm{total}}$ (Eq.~\ref{eq:full_obj}) uses $\lambda_{\mathrm{adv}}=\lambda_{\mathrm{id}}=\lambda_{\mathrm{comp}}=\lambda_{\mathrm{inv}}=1.0$, applied without warm-up, and $\lambda_{\mathrm{nt}}=0.1$ for the nontriviality penalty.

\paragraph{Seeds.}
All reported \method{} results are averaged over three random seeds ($\mathrm{seed}\in\{1,2,3\}$).

%% file: appendix/evaluation_all_envs.tex
\FloatBarrier
\section{Additional Symmetry-Discovery Results}
\label{appendix:additional_env_results}

We first report per-task errors on the symmetry-quality metrics and ablation results, then examine their dependence on the sampled group elements. We conclude with visualizations of the coupled observation--action transformations and alignment of transition-tuple embeddings during training. All results follow the metric definitions and evaluation protocol in Appendix~\ref{sec:eval}.

\subsection{Quantitative Comparison across Tasks}
\label{appendix:quantitative_all_envs}

Tables~\ref{tab:asi_all_tasks}--\ref{tab:invertibility_error_all_tasks} compare all methods on the four Robosuite tasks. Symmetry invariance and invertibility errors are averaged over $256$ uniformly sampled angles in $SO(2)$; compatibility error is averaged over all $256$ ordered pairs of $16$ equally spaced angles; identity error is evaluated at $g=e$. Lower is better for all four metrics.

\subsubsection{Symmetry invariance error}
\label{appendix:sym_all_envs}

\method{} attains the lowest aggregate symmetry invariance error and the best per-task means on Cube Dropping and Hammer Lifting (Table~\ref{tab:asi_all_tasks}). SymGAN and w/o $\mathcal L_{\mathrm{inv}}$ obtain the lowest means on Drawer Opening and Drawer Closing, respectively. Removing $\mathcal L_{\mathrm{adv}}$ yields the largest error on every task.

\begin{table}[!htbp]
  \centering
  \caption{\textbf{Average symmetry invariance error} $\overline{\mathcal E}_{\mathrm{sym}}$. Each task entry reports mean $\pm$ standard deviation across three training seeds, after averaging over $256$ evaluation angles within each seed. The All Tasks row reports the mean and standard deviation across all twelve task--seed averages. Lower is better.}
  \label{tab:asi_all_tasks}
  \setlength{\tabcolsep}{2pt}
  \renewcommand{\arraystretch}{1.05}
  \resizebox{\columnwidth}{!}{%
    \begin{tabular}{l|ccccccc}
      \toprule
      Task           & \method{}                  & SymGAN                     & LaLiGAN           & w/o $\mathcal{L}_{\mathrm{adv}}$ & w/o $\mathcal{L}_{\mathrm{id}}$ & w/o $\mathcal{L}_{\mathrm{comp}}$ & w/o $\mathcal{L}_{\mathrm{inv}}$ \\
      \midrule
      Drawer Opening & $0.198 \pm 0.021$          & $\mathbf{0.180 \pm 0.035}$ & $0.276 \pm 0.060$ & $1.794 \pm 0.114$                & $0.208 \pm 0.021$               & $0.223 \pm 0.029$                 & $0.225 \pm 0.011$                \\
      Drawer Closing & $0.174 \pm 0.005$          & $0.233 \pm 0.042$          & $0.498 \pm 0.087$ & $1.942 \pm 0.035$                & $0.171 \pm 0.026$               & $0.179 \pm 0.003$                 & $\mathbf{0.145 \pm 0.015}$       \\
      Cube Dropping  & $\mathbf{0.124 \pm 0.023}$ & $0.213 \pm 0.012$          & $0.208 \pm 0.009$ & $1.522 \pm 0.115$                & $0.131 \pm 0.023$               & $0.136 \pm 0.015$                 & $0.129 \pm 0.013$                \\
      Hammer Lifting & $\mathbf{0.131 \pm 0.011}$ & $0.322 \pm 0.025$          & $0.336 \pm 0.051$ & $1.546 \pm 0.035$                & $0.137 \pm 0.005$               & $0.159 \pm 0.002$                 & $0.158 \pm 0.004$                \\
      \midrule
      All Tasks      & $\mathbf{0.157 \pm 0.035}$ & $0.237 \pm 0.061$          & $0.329 \pm 0.122$ & $1.701 \pm 0.195$                & $0.162 \pm 0.037$               & $0.174 \pm 0.036$                 & $0.164 \pm 0.038$                \\
      \bottomrule
    \end{tabular}%
  }
\end{table}

\FloatBarrier

\subsubsection{Identity error}
\label{appendix:id_all_envs}

Removing $\mathcal L_{\mathrm{id}}$ increases identity error on all four tasks (Table~\ref{tab:identity_error_all_tasks}). Among \method{}, SymGAN, and LaLiGAN, LaLiGAN has the lowest aggregate identity error. The w/o $\mathcal L_{\mathrm{adv}}$ ablation obtains the lowest values overall, despite its large errors on the other three metrics.

\begin{table}[!htbp]
  \centering
  \captionsetup{singlelinecheck=false}
  \caption{\textbf{Average identity error} $\mathcal E_{\mathrm{id}}$. Mean $\pm$ standard deviation across tasks. Lower is better.}
  \label{tab:identity_error_all_tasks}
  \setlength{\tabcolsep}{2pt}
  \renewcommand{\arraystretch}{1.05}
  \resizebox{\columnwidth}{!}{%
    \begin{tabular}{l|ccccccc}
      \toprule
      Task           & \method{}         & SymGAN            & LaLiGAN           & w/o $\mathcal{L}_{\mathrm{adv}}$ & w/o $\mathcal{L}_{\mathrm{id}}$ & w/o $\mathcal{L}_{\mathrm{comp}}$ & w/o $\mathcal{L}_{\mathrm{inv}}$ \\
      \midrule
      Drawer Opening & $0.295 \pm 0.010$ & $1.790 \pm 0.062$ & $0.317 \pm 0.003$ & $\mathbf{0.157 \pm 0.007}$       & $0.361 \pm 0.024$               & $0.308 \pm 0.003$                 & $0.322 \pm 0.012$                \\
      Drawer Closing & $0.324 \pm 0.022$ & $2.049 \pm 0.005$ & $0.351 \pm 0.021$ & $\mathbf{0.162 \pm 0.003}$       & $0.372 \pm 0.072$               & $0.321 \pm 0.011$                 & $0.321 \pm 0.043$                \\
      Cube Dropping  & $0.344 \pm 0.050$ & $1.769 \pm 0.018$ & $0.331 \pm 0.013$ & $\mathbf{0.195 \pm 0.015}$       & $0.440 \pm 0.020$               & $0.347 \pm 0.007$                 & $0.377 \pm 0.026$                \\
      Hammer Lifting & $0.303 \pm 0.013$ & $2.208 \pm 0.176$ & $0.185 \pm 0.008$ & $\mathbf{0.111 \pm 0.001}$       & $0.757 \pm 0.018$               & $0.322 \pm 0.004$                 & $0.326 \pm 0.013$                \\
      \midrule
      All Tasks      & $0.317 \pm 0.034$ & $1.954 \pm 0.206$ & $0.296 \pm 0.066$ & $\mathbf{0.156 \pm 0.031}$       & $0.483 \pm 0.166$               & $0.325 \pm 0.016$                 & $0.337 \pm 0.036$                \\
      \bottomrule
    \end{tabular}%
  }
\end{table}

\FloatBarrier

\subsubsection{Compatibility error}
\label{appendix:comp_all_envs}

\method{} achieves the lowest aggregate compatibility error and the best per-task means on Drawer Opening, Cube Dropping, and Hammer Lifting (Table~\ref{tab:compatibility_error_all_tasks}). The w/o $\mathcal L_{\mathrm{inv}}$ ablation obtains the lowest mean on Drawer Closing. Removing $\mathcal L_{\mathrm{comp}}$ increases compatibility error on every task.

\begin{table}[!htbp]
  \centering
  \captionsetup{singlelinecheck=false}
  \caption{\textbf{Average compatibility error} $\overline{\mathcal E}_{\mathrm{comp}}$. Mean $\pm$ standard deviation across tasks. Lower is better.}
  \label{tab:compatibility_error_all_tasks}
  \setlength{\tabcolsep}{2pt}
  \renewcommand{\arraystretch}{1.05}
  \resizebox{\columnwidth}{!}{%
    \begin{tabular}{l|ccccccc}
      \toprule
      Task           & \method{}                  & SymGAN            & LaLiGAN           & w/o $\mathcal{L}_{\mathrm{adv}}$ & w/o $\mathcal{L}_{\mathrm{id}}$ & w/o $\mathcal{L}_{\mathrm{comp}}$ & w/o $\mathcal{L}_{\mathrm{inv}}$ \\
      \midrule
      Drawer Opening & $\mathbf{0.856 \pm 0.283}$ & $1.824 \pm 0.136$ & $1.853 \pm 0.155$ & $3.531 \pm 0.192$                & $1.009 \pm 0.221$               & $1.308 \pm 0.276$                 & $1.090 \pm 0.094$                \\
      Drawer Closing & $0.668 \pm 0.059$          & $2.121 \pm 0.109$ & $1.959 \pm 0.149$ & $3.314 \pm 0.120$                & $0.651 \pm 0.292$               & $0.871 \pm 0.030$                 & $\mathbf{0.445 \pm 0.155}$       \\
      Cube Dropping  & $\mathbf{0.521 \pm 0.160}$ & $1.888 \pm 0.052$ & $1.679 \pm 0.093$ & $3.718 \pm 0.364$                & $0.662 \pm 0.074$               & $0.913 \pm 0.204$                 & $0.614 \pm 0.272$                \\
      Hammer Lifting & $\mathbf{0.821 \pm 0.123}$ & $2.144 \pm 0.016$ & $1.861 \pm 0.007$ & $3.012 \pm 0.118$                & $0.951 \pm 0.078$               & $1.443 \pm 0.025$                 & $1.101 \pm 0.038$                \\
      \midrule
      All Tasks      & $\mathbf{0.717 \pm 0.221}$ & $1.994 \pm 0.167$ & $1.838 \pm 0.155$ & $3.394 \pm 0.344$                & $0.818 \pm 0.251$               & $1.134 \pm 0.301$                 & $0.813 \pm 0.333$                \\
      \bottomrule
    \end{tabular}%
  }
\end{table}

\FloatBarrier

\subsubsection{Invertibility error}
\label{appendix:inv_all_envs}

\method{} attains the lowest aggregate invertibility error and the best mean on Cube Dropping (Table~\ref{tab:invertibility_error_all_tasks}). SymGAN, w/o $\mathcal L_{\mathrm{inv}}$, and w/o $\mathcal L_{\mathrm{comp}}$ obtain the lowest means on Drawer Opening, Drawer Closing, and Hammer Lifting, respectively. Removing $\mathcal L_{\mathrm{inv}}$ increases invertibility error on three tasks but reduces it on Drawer Closing.

\begin{table}[!htbp]
  \centering
  \captionsetup{singlelinecheck=false}
  \caption{\textbf{Average invertibility error} $\overline{\mathcal E}_{\mathrm{inv}}$. Mean $\pm$ standard deviation across tasks. Lower is better.}
  \label{tab:invertibility_error_all_tasks}
  \setlength{\tabcolsep}{2pt}
  \renewcommand{\arraystretch}{1.05}
  \resizebox{\columnwidth}{!}{%
    \begin{tabular}{l|ccccccc}
      \toprule
      Task           & \method{}                  & SymGAN                     & LaLiGAN           & w/o $\mathcal{L}_{\mathrm{adv}}$ & w/o $\mathcal{L}_{\mathrm{id}}$ & w/o $\mathcal{L}_{\mathrm{comp}}$ & w/o $\mathcal{L}_{\mathrm{inv}}$ \\
      \midrule
      Drawer Opening & $0.911 \pm 0.242$          & $\mathbf{0.649 \pm 0.095}$ & $1.138 \pm 0.063$ & $5.947 \pm 0.369$                & $1.051 \pm 0.200$               & $1.083 \pm 0.347$                 & $1.115 \pm 0.105$                \\
      Drawer Closing & $0.712 \pm 0.057$          & $1.132 \pm 0.281$          & $1.927 \pm 0.208$ & $6.142 \pm 0.211$                & $0.689 \pm 0.229$               & $0.652 \pm 0.006$                 & $\mathbf{0.537 \pm 0.154}$       \\
      Cube Dropping  & $\mathbf{0.612 \pm 0.132}$ & $1.498 \pm 0.152$          & $1.360 \pm 0.114$ & $4.594 \pm 0.324$                & $0.771 \pm 0.065$               & $0.735 \pm 0.099$                 & $0.747 \pm 0.205$                \\
      Hammer Lifting & $0.685 \pm 0.081$          & $1.514 \pm 0.004$          & $1.480 \pm 0.056$ & $3.842 \pm 0.185$                & $0.744 \pm 0.053$               & $\mathbf{0.683 \pm 0.016}$        & $0.903 \pm 0.018$                \\
      \midrule
      All Tasks      & $\mathbf{0.730 \pm 0.183}$ & $1.198 \pm 0.389$          & $1.476 \pm 0.314$ & $5.131 \pm 0.995$                & $0.814 \pm 0.211$               & $0.788 \pm 0.250$                 & $0.826 \pm 0.253$                \\
      \bottomrule
    \end{tabular}%
  }
\end{table}

\FloatBarrier

\begin{samepage}
  \paragraph{Joint interpretation.}
  \method{} has the lowest aggregate symmetry invariance, compatibility, and invertibility errors. Removing the adversarial tuple-matching term improves identity but substantially worsens the other three metrics. Removing compatibility increases aggregate compatibility error more than removing identity or invertibility, consistent with Appendix~\ref{appendix:axiom_redundancy}. The task-level exceptions show that the benefits of individual regularizers depend on the task.

\end{samepage}

\FloatBarrier
\subsection{Symmetry Quality across Group Elements}
\label{appendix:angle_analysis}

The curves and heatmaps use the same evaluation angles and angle pairs as Section~\ref{appendix:quantitative_all_envs}, revealing whether errors concentrate near particular transformations.

\subsubsection{Symmetry invariance across angles}
\label{appendix:sym_angle_all_envs}

Figure~\ref{fig:sym_curve_all_envs} shows low, nearly flat symmetry invariance error for \method{} across the sampled angles on all four tasks. Removing $\mathcal L_{\mathrm{adv}}$ produces broad plateaus or peaks near $180^\circ$, revealing distribution mismatch across a wide range of transformations.

\begin{figure}[!htbp]
  \centering
  \begin{minipage}{0.49\linewidth}
    \centering
    \includegraphics[width=\linewidth]{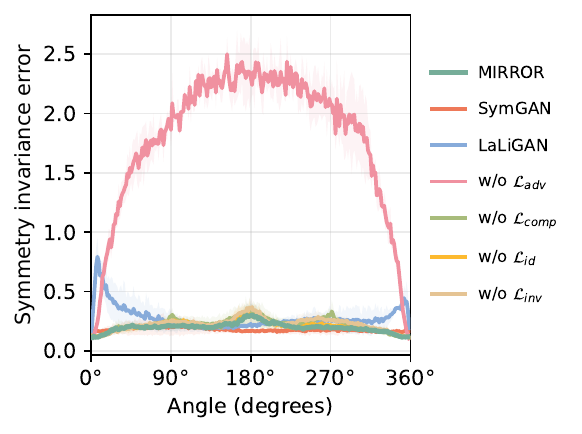}\\[2pt]
    \small (a) Drawer Opening
  \end{minipage}\hfill
  \begin{minipage}{0.49\linewidth}
    \centering
    \includegraphics[width=\linewidth]{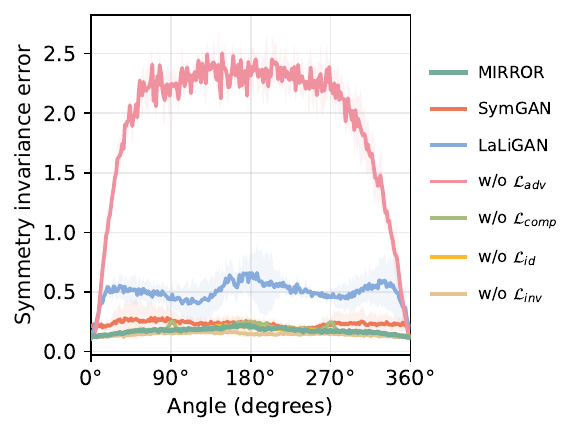}\\[2pt]
    \small (b) Drawer Closing
  \end{minipage}\\[6pt]
  \begin{minipage}{0.49\linewidth}
    \centering
    \includegraphics[width=\linewidth]{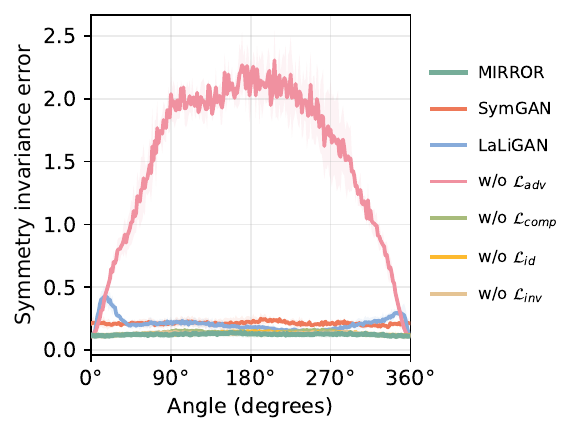}\\[2pt]
    \small (c) Cube Dropping
  \end{minipage}\hfill
  \begin{minipage}{0.49\linewidth}
    \centering
    \includegraphics[width=\linewidth]{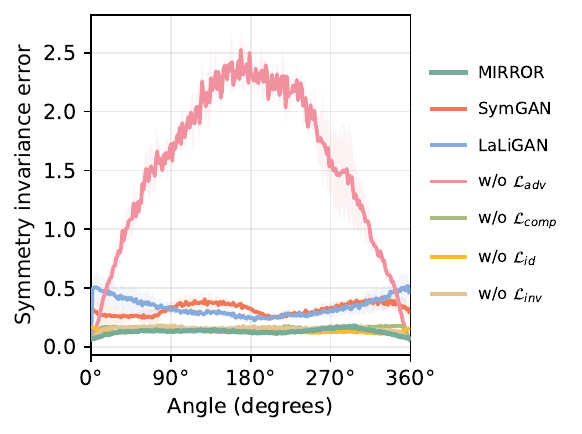}\\[2pt]
    \small (d) Hammer Lifting
  \end{minipage}
  \caption{
    \textbf{Symmetry invariance error across latent rotation angles.}
    Lines show the mean feature-space Wasserstein-2 distance between original and transformed transition-tuple distributions at $256$ uniformly sampled $SO(2)$ angles.
    Shaded bands indicate $\pm 1$ standard deviation across three seeds; lower is better.
    \method{} maintains low error across all four tasks.
  }
  \label{fig:sym_curve_all_envs}
\end{figure}

\clearpage
\subsubsection{Compatibility across Angle Pairs}
\label{appendix:comp_angle_all_envs}

Figure~\ref{fig:comp_heatmap_all_envs} shows low compatibility error for \method{} across the evaluated angle pairs. Removing $\mathcal L_{\mathrm{adv}}$ produces bands near inverse-angle pairs $\theta_{g_2}\approx-\theta_{g_1}\pmod{2\pi}$ or high-error regions near $\theta_{g_1}=180^\circ$. LaLiGAN shows narrower ridges or intersecting streaks, while SymGAN has more uniformly elevated error.

\begin{figure}[!htbp]
  \centering
  \includegraphics[width=\linewidth]{figures/fig_comp_error_drawer_open_v1.pdf}\\[2pt]
  \small (a) Drawer Opening\\[10pt]
  \includegraphics[width=\linewidth]{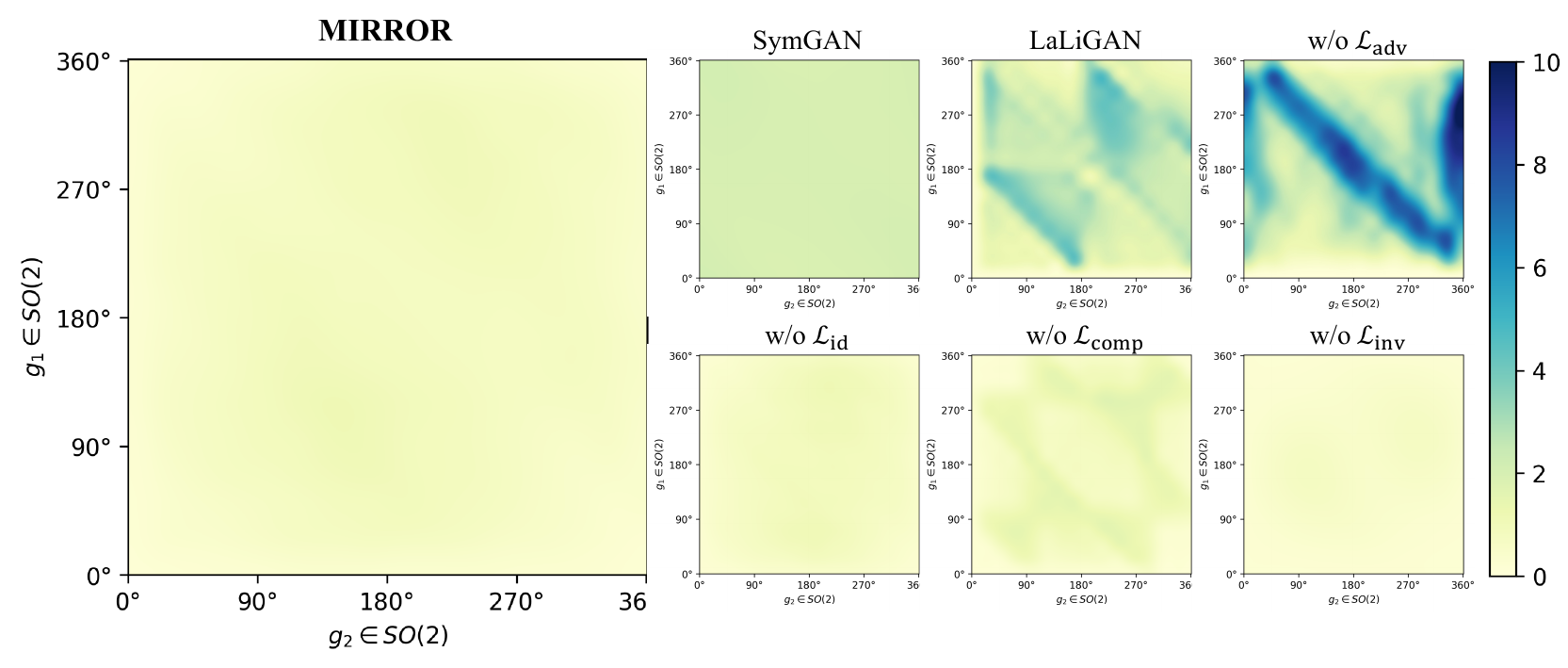}\\[2pt]
  \small (b) Drawer Closing
  \caption{
    \textbf{Compatibility error heatmaps across four tasks.}
    Compatibility error compares a transition tuple transformed successively by $g_1$ and $g_2$ with the same tuple transformed once by their composition $g_2\circ g_1$.
    Errors are evaluated at all $256$ ordered pairs of $16$ equally spaced rotation angles; the heatmaps are smoothed for visualization.
    Within each task panel, the large left subplot shows \method{} and the six smaller subplots show baselines and single-loss ablations on a shared color scale (lower is better).
    \method{} maintains low error across the evaluated angle pairs. Panels (c)--(d) continue on the next page.
  }
  \label{fig:comp_heatmap_all_envs}
\end{figure}

\clearpage
\begin{figure}[!htbp]\ContinuedFloat
  \centering
  \includegraphics[width=\linewidth]{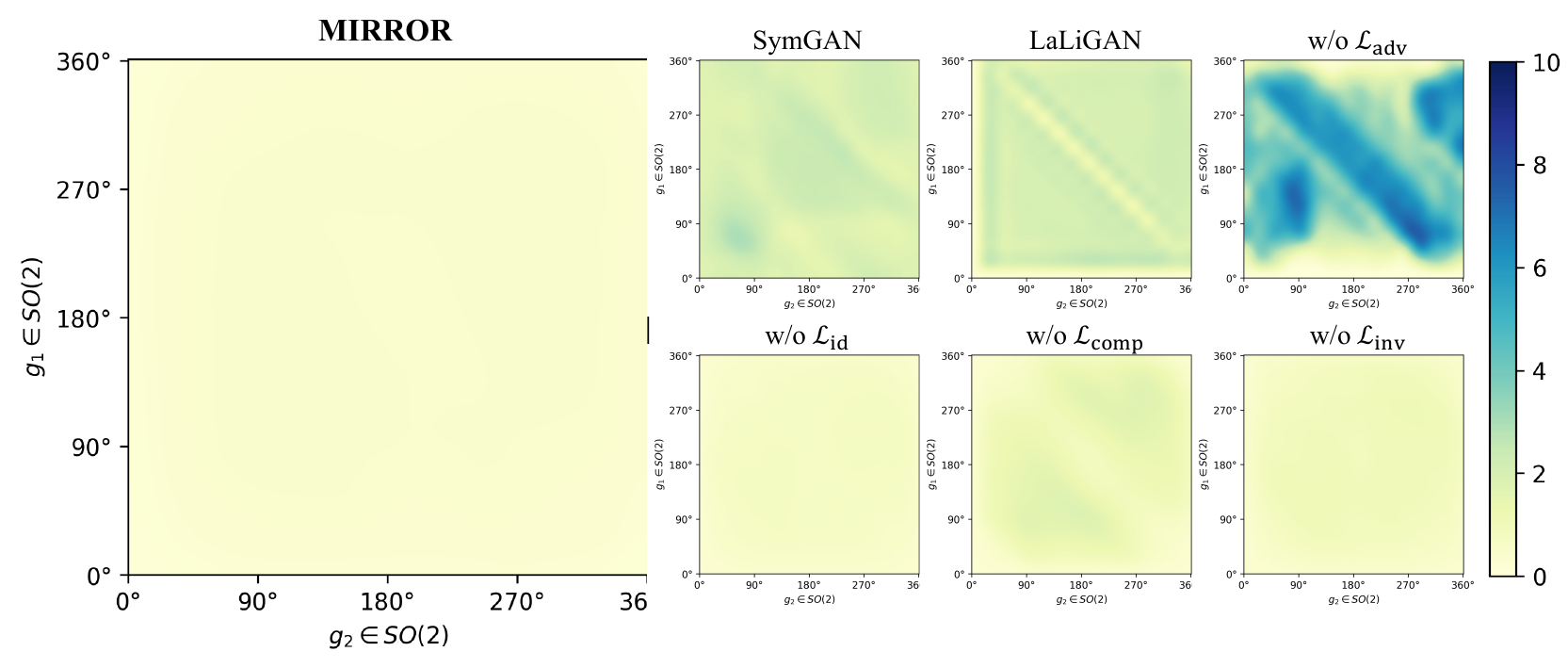}\\[2pt]
  \small (c) Cube Dropping\\[10pt]
  \includegraphics[width=\linewidth]{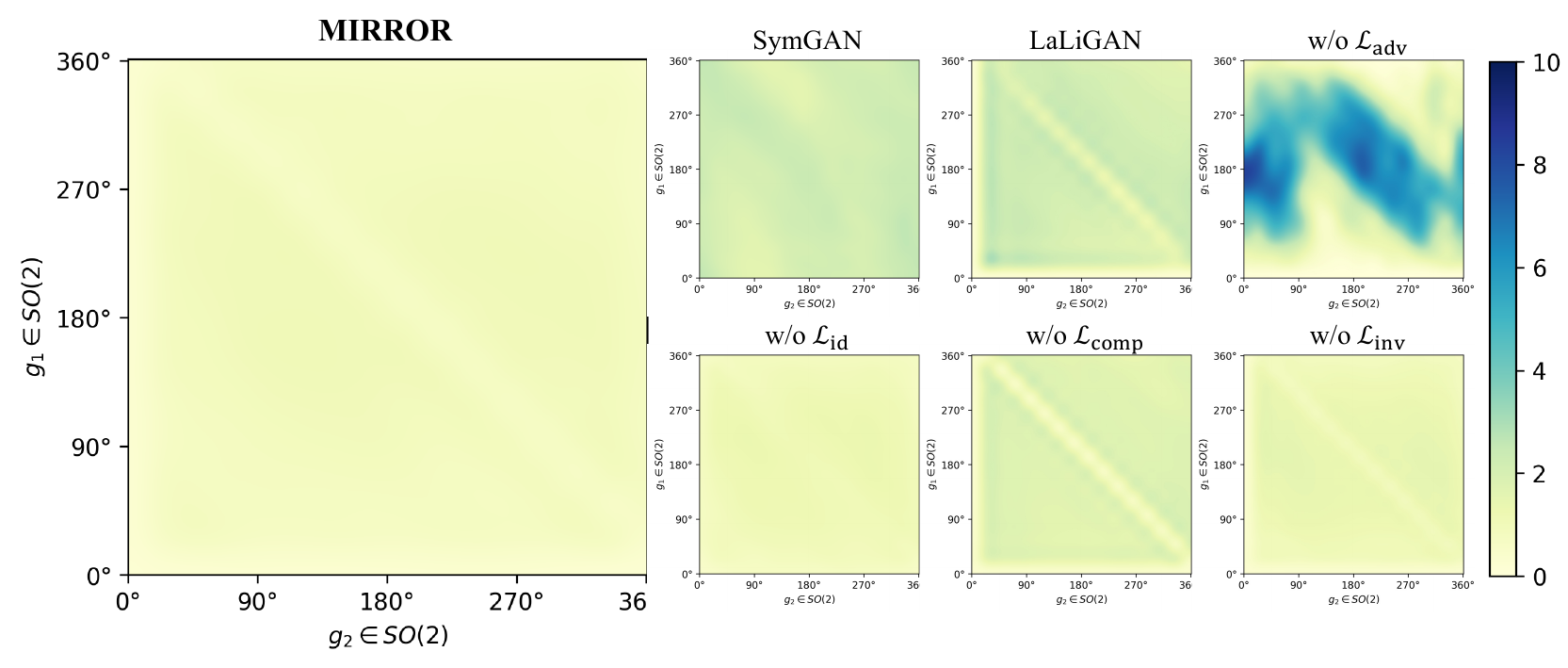}\\[2pt]
  \small (d) Hammer Lifting
  \caption{
    \textbf{Compatibility error heatmaps across four tasks (continued).}
    Panels (c)--(d) show Cube Dropping and Hammer Lifting.
    The smoothed heatmaps compare successive transformations by $g_1$ and $g_2$ with their composition $g_2\circ g_1$, evaluated at all $256$ ordered pairs of $16$ equally spaced rotation angles.
    Within each task panel, \method{}, baselines, and single-loss ablations share a color scale. Lower is better.
  }
\end{figure}

\clearpage

\subsubsection{Invertibility across Angles}
\label{appendix:inv_angle_all_envs}

Figure~\ref{fig:inv_curve_all_envs} shows that \method{} maintains relatively low invertibility error across the sampled angles on all four tasks. Removing $\mathcal L_{\mathrm{adv}}$ produces broad plateaus or peaks near $180^\circ$, indicating that invertibility also degrades across a wide range of transformations when adversarial tuple matching is removed. The remaining methods exhibit task-dependent differences, as summarized in Table~\ref{tab:invertibility_error_all_tasks}.

\begin{figure}[!htbp]
  \centering
  \begin{minipage}{0.49\linewidth}
    \centering
    \includegraphics[width=\linewidth]{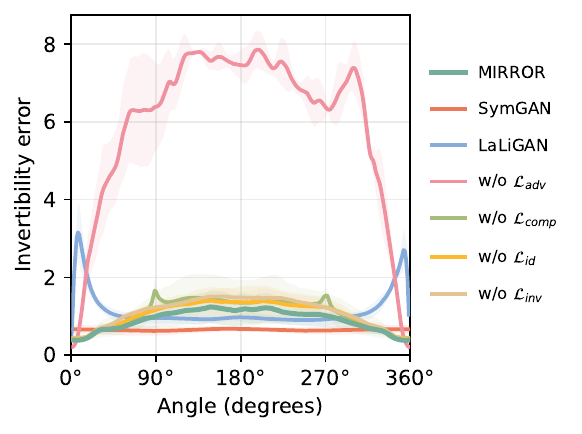}\\[2pt]
    \small (a) Drawer Opening
  \end{minipage}\hfill
  \begin{minipage}{0.49\linewidth}
    \centering
    \includegraphics[width=\linewidth]{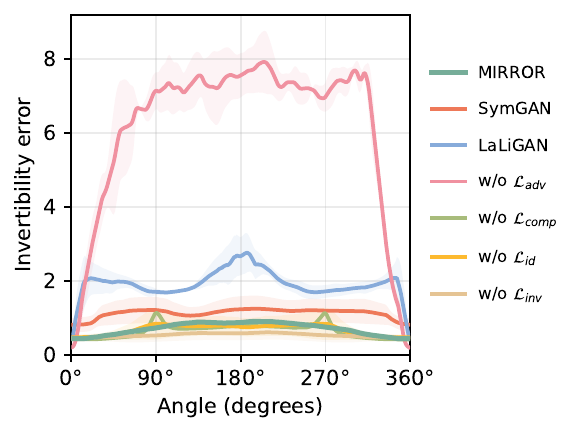}\\[2pt]
    \small (b) Drawer Closing
  \end{minipage}\\[6pt]
  \begin{minipage}{0.49\linewidth}
    \centering
    \includegraphics[width=\linewidth]{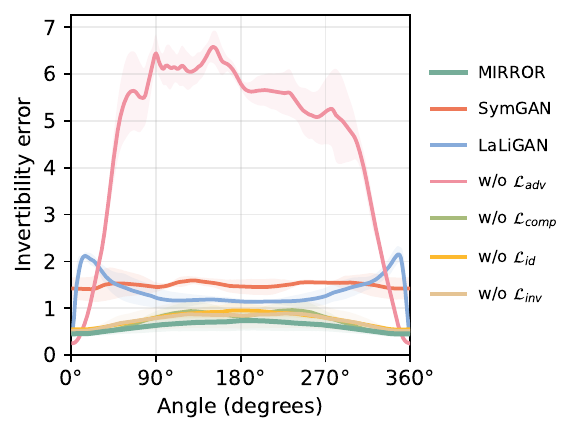}\\[2pt]
    \small (c) Cube Dropping
  \end{minipage}\hfill
  \begin{minipage}{0.49\linewidth}
    \centering
    \includegraphics[width=\linewidth]{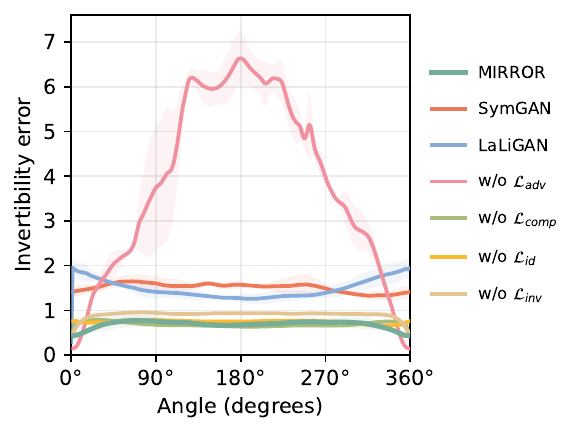}\\[2pt]
    \small (d) Hammer Lifting
  \end{minipage}
  \caption{
    \textbf{Invertibility error across latent rotation angles.}
    For each task, $\mathcal E_{\mathrm{inv}}(g)$ measures how closely applying $g$ followed by $g^{-1}$ recovers the original transition tuple.
    Lines show mean errors at $256$ angles uniformly sampled from $SO(2)$, and shaded bands indicate $\pm 1$ standard deviation across three training seeds. Lower is better.
    \method{} maintains relatively low recovery error across the sampled angles on all four tasks.
  }
  \label{fig:inv_curve_all_envs}
\end{figure}

\clearpage
\subsection{Qualitative Visualization of Discovered Symmetries}
\label{appendix:qualitative_all_envs}

\subsubsection{Learned Observation--Action Transformations}

Figure~\ref{fig:vis_symmetry_all_envs} applies each $g\in C_4$ to the same observation--action pair $(o,a)$ in each environment. Across all four tasks, $\phi_o(g,o)$ changes the scene configuration while $\phi_a(g,a)$ adjusts the normalized planar commands $(\Delta x,\Delta y)$ jointly with it. The remaining action components ($\Delta z$ and the task-specific $\Delta\psi$ and/or $\Delta w$) remain approximately unchanged. The angles specify rotations in the learned symmetry latents; their geometric effect after decoding is determined by the learned maps.

\begin{figure}[!htbp]
  \centering
  \includegraphics[width=\linewidth]{figures/fig_vis_symmetry_drawer_open_v1.pdf}\\[2pt]
  \small (a) Drawer Opening\\[6pt]
  \includegraphics[width=\linewidth]{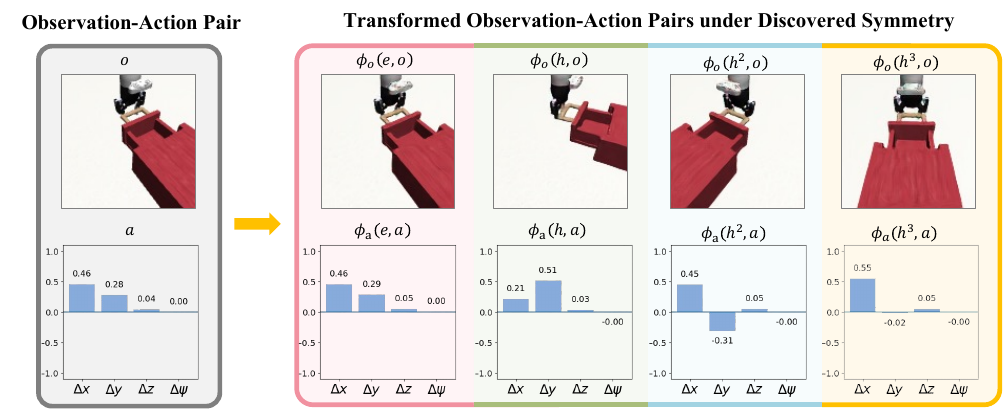}\\[2pt]
  \small (b) Drawer Closing
  \caption{
    \textbf{Learned RGB observation--action transformations.}
    Each task panel shows a fixed original pair $(o,a)$ on the left and four transformed pairs $(\phi_o(g,o),\phi_a(g,a))$ on the right, for $g\in C_4=\{e,h,h^2,h^3\}$ with latent angles $0^\circ,90^\circ,180^\circ,270^\circ$, respectively.
    Action bars show normalized command components on a $[-1,1]$ scale.
    The scene and planar commands $(\Delta x,\Delta y)$ change jointly; the decoded transformations need not be geometric rotations by the latent angles.
    Panels (c)--(d) continue on the next page.
  }
  \label{fig:vis_symmetry_all_envs}
\end{figure}

\begin{figure}[t!]\ContinuedFloat
  \centering
  \includegraphics[width=\linewidth]{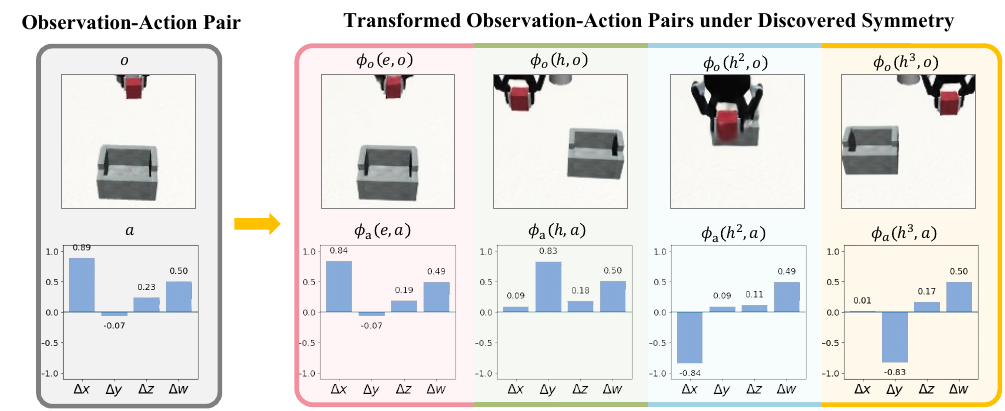}\\[2pt]
  \small (c) Cube Dropping\\[6pt]
  \includegraphics[width=\linewidth]{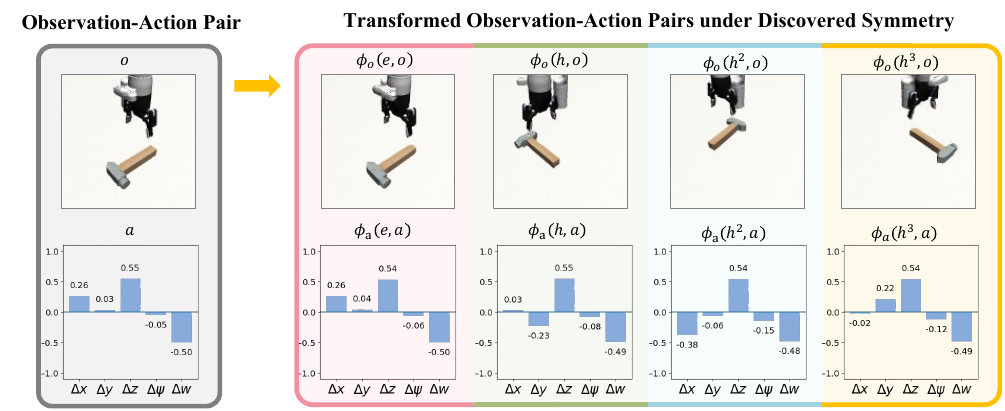}\\[2pt]
  \small (d) Hammer Lifting
  \caption{\textbf{Learned RGB observation--action transformations (continued).}
    Panels (c)--(d) show Cube Dropping and Hammer Lifting, with the original pair on the left and four transformed pairs on the right.
    Columns $e,h,h^2,h^3$ correspond to latent angles $0^\circ,90^\circ,180^\circ,270^\circ$, respectively; decoded transformations need not be geometric rotations by these angles.
    Action bars show normalized command components on a $[-1,1]$ scale, with $\Delta w$ denoting gripper opening.}
\end{figure}

\clearpage
\subsubsection{Alignment of Transition-Tuple Embeddings during Training}

Figure~\ref{fig:latent_alignment_all_envs} projects transition-tuple embeddings of $\mathcal D$ and $\Phi(g,\mathcal D)$ to two dimensions with PCA, separately for each $g\in C_4$ and at four training checkpoints. The projected clouds are initially separated for nonidentity elements and increasingly overlap during training. By $1{\times}10^5$ steps, they overlap substantially across the four tasks. This visualization supports the quantitative distribution-matching results in Table~\ref{tab:asi_all_tasks}; overlap in a two-dimensional projection alone does not establish equality of the full transition-tuple distributions.

\begin{figure}[!htbp]
  \centering
  \begin{minipage}{0.49\linewidth}
    \centering
    \includegraphics[width=\linewidth]{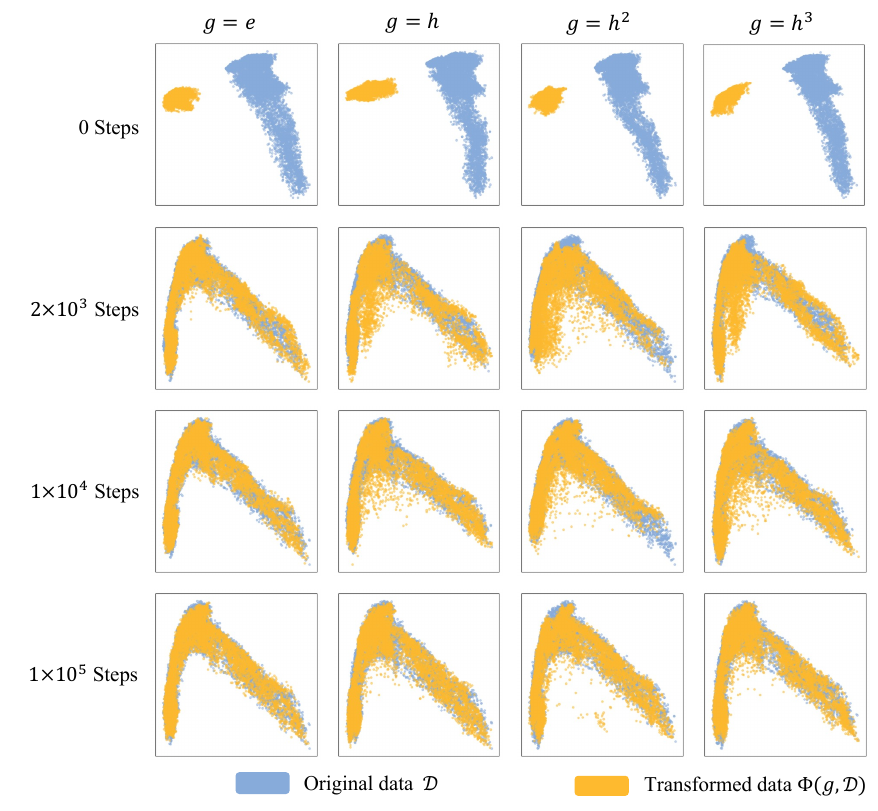}\\[2pt]
    \small (a) Drawer Opening
  \end{minipage}\hfill
  \begin{minipage}{0.49\linewidth}
    \centering
    \includegraphics[width=\linewidth]{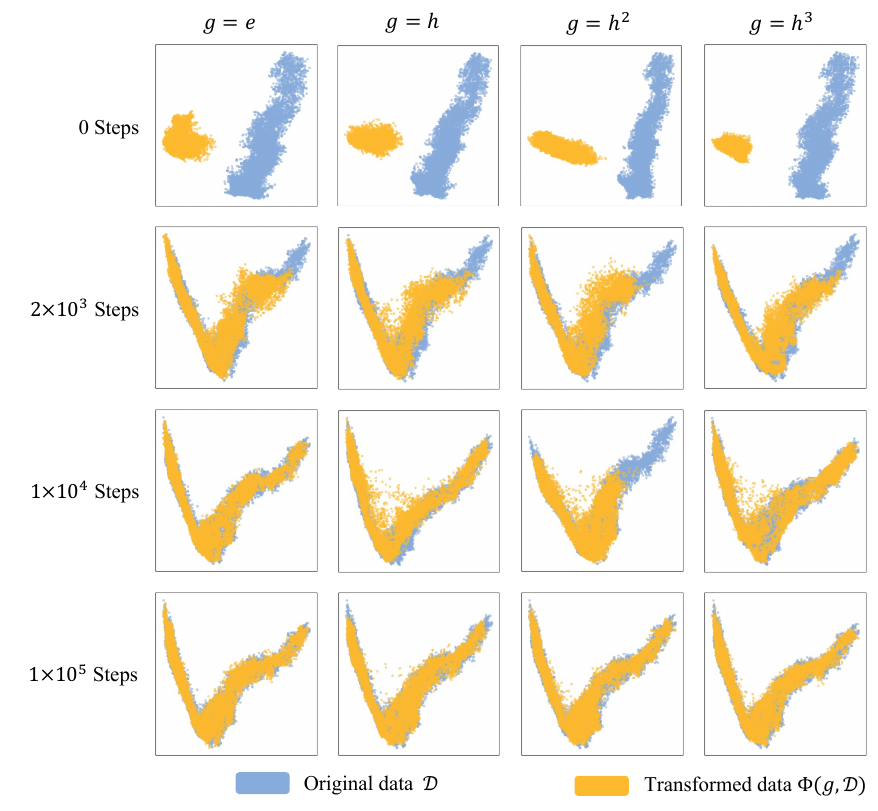}\\[2pt]
    \small (b) Drawer Closing
  \end{minipage}\\[6pt]
  \begin{minipage}{0.49\linewidth}
    \centering
    \includegraphics[width=\linewidth]{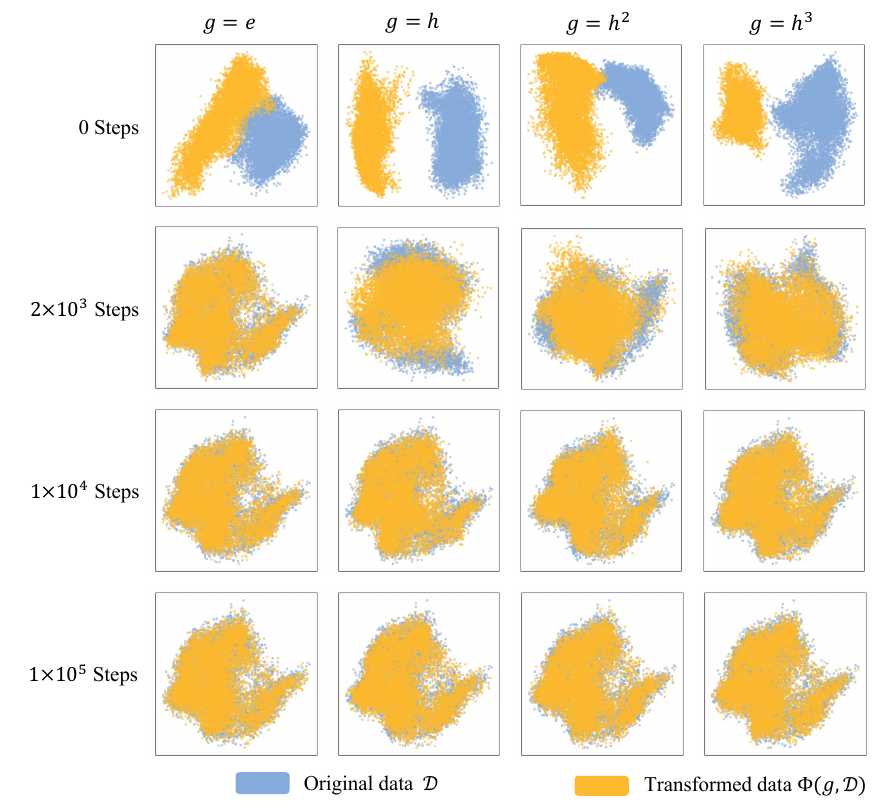}\\[2pt]
    \small (c) Cube Dropping
  \end{minipage}\hfill
  \begin{minipage}{0.49\linewidth}
    \centering
    \includegraphics[width=\linewidth]{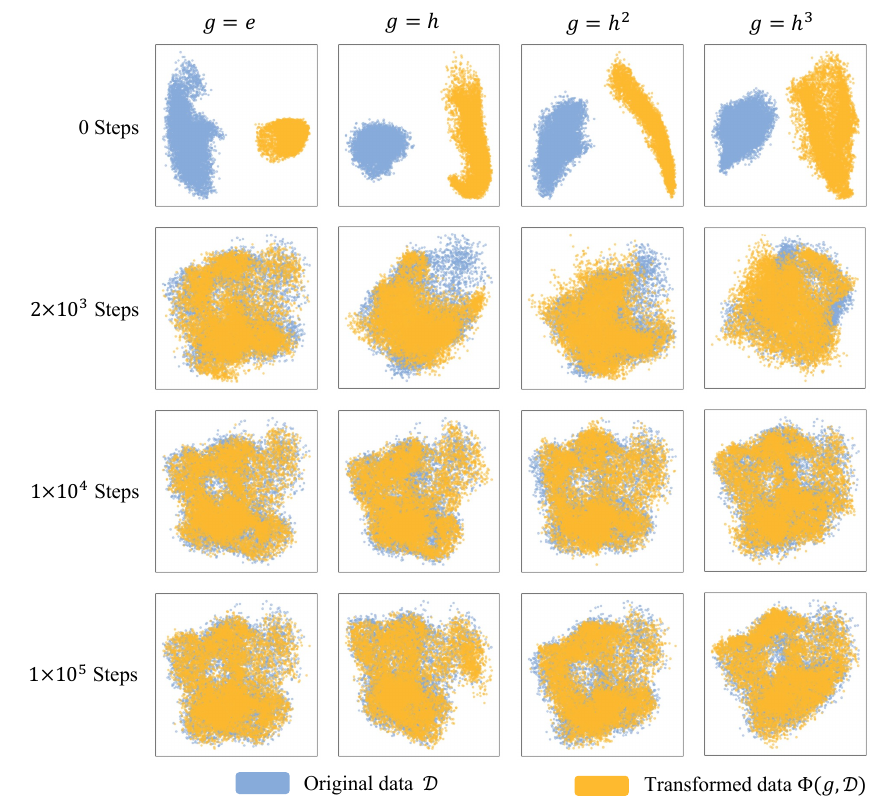}\\[2pt]
    \small (d) Hammer Lifting
  \end{minipage}
  \caption{
    \textbf{Alignment of transition-tuple embeddings during training.}
    Two-dimensional PCA projections show embeddings of the original dataset $\mathcal D$ (blue) and transformed dataset $\Phi(g,\mathcal D)$ (orange) across four tasks.
    Columns correspond to $g\in C_4=\{e,h,h^2,h^3\}$, where $h$ denotes a $90^\circ$ latent rotation; rows show $0$, $2{\times}10^3$, $1{\times}10^4$, and $1{\times}10^5$ training steps.
    Increasing overlap qualitatively supports distribution matching.
  }
  \label{fig:latent_alignment_all_envs}
\end{figure}

\FloatBarrier

%% file: appendix/equivariant_bc.tex
\FloatBarrier
\section{Frame-Averaged Behavior Cloning}
\label{appendix:equivariant_bc}

This appendix describes EBC-\method{}, the frame-averaged BC policy built on the discovered transformations $(\phi_o,\phi_a)$. We present the policy construction and its equivariance conditions (Section~\ref{appendix:frame_averaging}), architecture and training objective (Section~\ref{appendix:ebc_training}), baselines and evaluation protocol (Section~\ref{appendix:ebc_eval}), and subgroup-size results (Section~\ref{appendix:ebc_subgroup_results}). The RGB results support Section~\ref{sec:application}; the sim-to-real experiments are collected in Appendix~\ref{appendix:sim2real}.

\subsection{Frame-Averaged Policy and Equivariance Conditions}
\label{appendix:frame_averaging}

Let $H$ be a finite subgroup of $G$ and let $f_\theta$ be a base policy network that maps observations to actions. Following the frame-averaging construction~\citep{puny2021frame}, we use the discovered transformations $\phi_o$ and $\phi_a$ to define
\begin{equation}
  \pi_\theta(o)
  ~:=~
  \frac{1}{|H|}\sum_{g\in H}
  \phi_a\!\left(g^{-1},\, f_\theta(\phi_o(g,o))\right).
  \label{eq:fa_policy}
\end{equation}
Intuitively, we (i) transform the input observation by $g$, (ii) predict an action in that transformed frame via $f_\theta$, (iii) map the predicted action back with $\phi_a(g^{-1},\cdot)$, and finally (iv) average over $g\in H$.

Our experiments use $G=SO(2)$ and average over a finite cyclic subgroup $H=C_n\subset SO(2)$. The equivariance result below concerns this subgroup $H$.

\paragraph{Conditions for exact equivariance.}
Assume that $\phi_o$ and $\phi_a$ satisfy the group-operation axioms exactly on the relevant supports, and that all transformations and averages below are well-defined. We additionally require the action transformation to commute with arithmetic averaging: for every $h\in H$ and admissible collection of actions $\{a_k\}_{k\in H}$,
\begin{equation}
  \phi_a\!\left(h,\frac{1}{|H|}\sum_{k\in H}a_k\right)
  =\frac{1}{|H|}\sum_{k\in H}\phi_a(h,a_k).
  \label{eq:fa_mean_compatibility}
\end{equation}
This condition holds, for example, when each $\phi_a(h,\cdot)$ is affine and the averages remain in the action domain. Under these assumptions, reindexing the sum by $k=g\circ h$ gives
\begin{align*}
  \pi_\theta(\phi_o(h,o))
   & =\frac{1}{|H|}\sum_{g\in H}\phi_a\!\left(g^{-1},f_\theta(\phi_o(g\circ h,o))\right)                 \\
   & =\frac{1}{|H|}\sum_{k\in H}\phi_a\!\left(h,\phi_a\!\left(k^{-1},f_\theta(\phi_o(k,o))\right)\right) \\
   & =\phi_a(h,\pi_\theta(o)).
\end{align*}
The first two equalities use the exact group-operation axioms, and the last uses Eq.~\eqref{eq:fa_mean_compatibility}. Thus, Eq.~\eqref{eq:fa_policy} defines an $H$-equivariant policy under the stated conditions.

\paragraph{Scope for learned transformations.}
In \method{}, $\phi_a(g,\cdot)=D_a(\rho_a(g)E_a(\cdot))$ is nonlinear, and the learned coupled observation--action transformations are regularized to satisfy the group-operation axioms. Neither exact group consistency nor Eq.~\eqref{eq:fa_mean_compatibility} is guaranteed by this parameterization. We assess the resulting frame-averaged BC policy, EBC-\method{}, through downstream policy success rates; the exact-equivariance guarantee above applies under its stated assumptions.

\subsection{Policy Architecture and Training Objective}
\label{appendix:ebc_training}

We instantiate $f_\theta$ with a pretrained ResNet-18 backbone followed by a two-layer MLP with 64 hidden units per layer and ReLU activations. The same base policy is applied to every transformed observation in Eq.~\eqref{eq:fa_policy}.

During BC training, we hold the perceptual autoencoder and the observation and action symmetry modules fixed and optimize the base-policy parameters $\theta$. The mean-squared BC loss is applied to the final averaged action:
\begin{equation}
  \min_\theta~ \mathcal L_{\mathrm{EBC}}(\theta)
  ~:=~
  \mathbb E_{(o,a)\sim\mathcal D}\Big[\,
    \|\pi_\theta(o)-a\|_2^2
    \,\Big].
  \label{eq:ebc_loss}
\end{equation}

Gradients propagate through the fixed action transformations to train $f_\theta$ within the frame-averaged policy.

\subsection{Baselines and Evaluation Protocol}
\label{appendix:ebc_eval}

We evaluate policies by rolling them out on the four RGB Robosuite tasks described in Appendix~\ref{appendix:env_description} and report per-task episodic success rates and their average.
We compare vanilla \textbf{BC}, which omits frame averaging, with four frame-averaged BC variants that use different coupled observation--action transformations. \textbf{EBC-\method} uses our discovered transformations; \textbf{EBC-SymGAN} and \textbf{EBC-LaLiGAN} reuse the GAN-discovered transformations from Section~\ref{sec:exp_setup}; and \textbf{EBC-ExSym} uses the hand-designed \emph{Extrapolated Symmetry}~\citep{wang2023surprising}.
ExSym rotates images in pixel space and applies the following transformation directly to normalized action coordinates:
\begin{equation*}
  (\Delta x,\Delta y,\Delta z,\Delta\psi,\Delta w)\;\mapsto\;
  \big(R(\theta_g)(\Delta x,\Delta y)^{\!\top},\,\Delta z,\,\Delta\psi,\,\Delta w\big),
\end{equation*}
where $R(\theta_g)$ is the $2$D rotation matrix. This rotates only the planar translation components; $\Delta\psi$ or $\Delta w$ is omitted when the task has no yaw or gripper control.
For all frame-averaged BC variants we sweep the subgroup size $n=|H|$ with $H=C_n$ and $n\in\{2,4,8,16\}$.

\subsection{Effect of Subgroup Size}
\label{appendix:ebc_subgroup_results}

Table~\ref{tab:ebc_results} reports the complete RGB subgroup sweep for $H\in\{C_2,C_4,C_8,C_{16}\}$, including the $C_{16}$ results from main paper Table~\ref{tab:ebc_results_main}. The average success rate of EBC-\method{} increases from $80.19\%$ at $C_2$ to $81.91\%$ at $C_4$, $83.57\%$ at $C_8$, and $85.12\%$ at $C_{16}$, exceeding vanilla BC ($80.48\%$) for $|H|\geq4$. Thus, $C_{16}$ achieves the highest average success rate among the tested subgroup sizes, although the best subgroup varies by task.

Averaging over more group elements requires more transformed policy evaluations per action. Its performance benefit also depends on the transformations: EBC-ExSym, EBC-SymGAN, and EBC-LaLiGAN remain below vanilla BC on average at every tested subgroup size. These comparisons support the utility of \method{}'s learned transformations for frame averaging.

\begin{table}[htbp]
  \centering
  \caption{Episodic success rates in percent (mean $\pm$ std) for $H \in \{C_2, C_4, C_8, C_{16}\}$, with vanilla BC as reference. The $C_{16}$ rows reproduce main paper Table~\ref{tab:ebc_results_main}. \textbf{Bold} indicates the highest mean in each column across all listed methods and subgroup sizes, including BC.}
  \label{tab:ebc_results}
  \begin{adjustbox}{width=1.0\linewidth}
    \begin{tabular}{l|c|cccc|c}
      \toprule
      Method                                            & $H$      & \textbf{Drawer Opening} & \textbf{Drawer Closing} & \textbf{Cube Dropping}  & \textbf{Hammer Lifting} & \textbf{Average}        \\ \midrule
      \multirow{4}{*}{EBC-ExSym}                        & $C_{2}$  & $62.5 \pm 2.1$          & $85.2 \pm 1.4$          & $79.5 \pm 4.5$          & $84.8 \pm 4.3$          & $78.0 \pm 9.8$          \\
                                                        & $C_{4}$  & $61.5 \pm 2.6$          & $86.1 \pm 3.5$          & $83.5 \pm 4.7$          & $87.4 \pm 2.2$          & $79.6 \pm 11.1$         \\
                                                        & $C_{8}$  & $62.1 \pm 0.8$          & $84.5 \pm 5.0$          & $79.5 \pm 4.6$          & $85.3 \pm 5.6$          & $77.9 \pm 10.4$         \\
                                                        & $C_{16}$ & $64.7 \pm 2.2$          & $86.4 \pm 4.7$          & $80.9 \pm 5.9$          & $85.1 \pm 5.7$          & $79.3 \pm 9.9$          \\ \midrule
      \multirow{4}{*}{EBC-SymGAN}                       & $C_{2}$  & $69.7 \pm 7.5$          & $87.0 \pm 2.2$          & $79.6 \pm 7.9$          & $81.7 \pm 2.7$          & $79.5 \pm 8.5$          \\
                                                        & $C_{4}$  & $67.5 \pm 8.0$          & $86.6 \pm 2.6$          & $83.5 \pm 4.0$          & $68.1 \pm 4.9$          & $76.4 \pm 10.2$         \\
                                                        & $C_{8}$  & $68.7 \pm 8.7$          & $83.7 \pm 5.0$          & $84.1 \pm 5.1$          & $57.5 \pm 6.7$          & $73.5 \pm 12.9$         \\
                                                        & $C_{16}$ & $66.3 \pm 11.3$         & $85.8 \pm 3.6$          & $81.9 \pm 9.4$          & $49.8 \pm 16.3$         & $71.0 \pm 18.1$         \\ \midrule
      \multirow{4}{*}{EBC-LaLiGAN}                      & $C_{2}$  & $69.0 \pm 1.0$          & $80.7 \pm 6.6$          & $85.1 \pm 5.5$          & $82.5 \pm 0.8$          & $79.3 \pm 7.5$          \\
                                                        & $C_{4}$  & $69.0 \pm 3.1$          & $80.2 \pm 8.9$          & $79.5 \pm 8.3$          & $69.9 \pm 10.1$         & $74.7 \pm 9.6$          \\
                                                        & $C_{8}$  & $67.9 \pm 1.5$          & $73.3 \pm 3.0$          & $83.1 \pm 9.7$          & $56.6 \pm 2.4$          & $70.2 \pm 10.9$         \\
                                                        & $C_{16}$ & $65.3 \pm 2.0$          & $73.6 \pm 6.2$          & $84.5 \pm 6.2$          & $61.3 \pm 5.6$          & $71.2 \pm 10.3$         \\ \midrule
      \multicolumn{1}{l|}{\multirow{4}{*}{EBC-\method}} & $C_{2}$  & $66.3 \pm 1.6$          & $\mathbf{90.1 \pm 3.6}$ & $77.8 \pm 4.2$          & $86.5 \pm 2.6$          & $80.2 \pm 9.7$          \\
      \multicolumn{1}{l|}{}                             & $C_{4}$  & $65.2 \pm 2.5$          & $86.9 \pm 1.3$          & $86.0 \pm 3.5$          & $\mathbf{89.6 \pm 2.0}$ & $81.9 \pm 10.1$         \\
      \multicolumn{1}{l|}{}                             & $C_{8}$  & $69.3 \pm 3.1$          & $90.0 \pm 0.6$          & $87.8 \pm 3.5$          & $87.2 \pm 3.4$          & $83.6 \pm 8.8$          \\
                                                        & $C_{16}$ & $\mathbf{71.7 \pm 1.4}$ & $89.8 \pm 2.0$          & $\mathbf{90.3 \pm 2.9}$ & $88.7 \pm 2.5$          & $\mathbf{85.1 \pm 8.1}$ \\ \midrule
      BC                                                & -        & $71.4 \pm 0.8$          & $88.1 \pm 5.9$          & $80.6 \pm 2.5$          & $81.8 \pm 3.0$          & $80.5 \pm 7.0$          \\ \bottomrule
    \end{tabular}
  \end{adjustbox}
\end{table}

%% file: appendix/sim2real.tex
\FloatBarrier
\section{Sim-to-Real Experiments}
\label{appendix:sim2real}

This appendix details the sim-to-real experiments in Section~\ref{sec:sim2real}. We describe the hardware and observation preprocessing, specify the training and deployment protocol, and then report simulation and real-world results. Qualitative visualizations of the learned observation--action transformations follow the quantitative evaluation.

\input{appendix/env_sim2real}
\FloatBarrier
\input{appendix/evaluation_sim2real}
\FloatBarrier
\input{appendix/real_world_validation}
\clearpage
\input{appendix/sim2real_symmetry_vis}

%% file: appendix/env_sim2real.tex
\subsection{Hardware and Observation Preprocessing}
\label{appendix:real_env_description}

\subsubsection{Hardware Setup}
\label{appendix:real_hardware}

We deploy the learned policy on a $7$-DoF Kinova GEN3 manipulator equipped with a Robotiq 2F85 parallel-jaw gripper.
Visual observations are provided by an Intel RealSense D435 RGB-D camera mounted on a fixed stand in front of the workspace.
The camera-to-robot extrinsics are obtained once via standard eye-to-hand calibration and reused across all tasks.

\subsubsection{Observation Preprocessing}
\label{appendix:real_obs_preprocess}

The policy consumes a heightmap and a target-object segmentation mask rather than raw RGB.
Figure~\ref{fig:real_obs} illustrates the preprocessing pipeline, and main-paper Figure~\ref{fig:real_tasks_obs}(e,f) compares the resulting real-world observation against its simulated counterpart.

\paragraph{Heightmap.}
We project the depth image into a top-down heightmap of the workspace and clip the height values to the range $[0,\,0.3]\,\text{m}$ above the table surface.
The clipped heights are then linearly rescaled to $[0,\,255]$ and stored as an $8$-bit single-channel image, matching the dynamic range used in simulation.

\paragraph{Segmentation mask.}
The target-object segmentation mask is obtained with SAM 2~\citep{ravi2024sam2}, prompted using a height-derived foreground prior.
The resulting binary segmentation mask is stacked with the heightmap to form the final policy input.

\begin{figure}[htbp]
  \centering
  \includegraphics[width=0.9\linewidth]{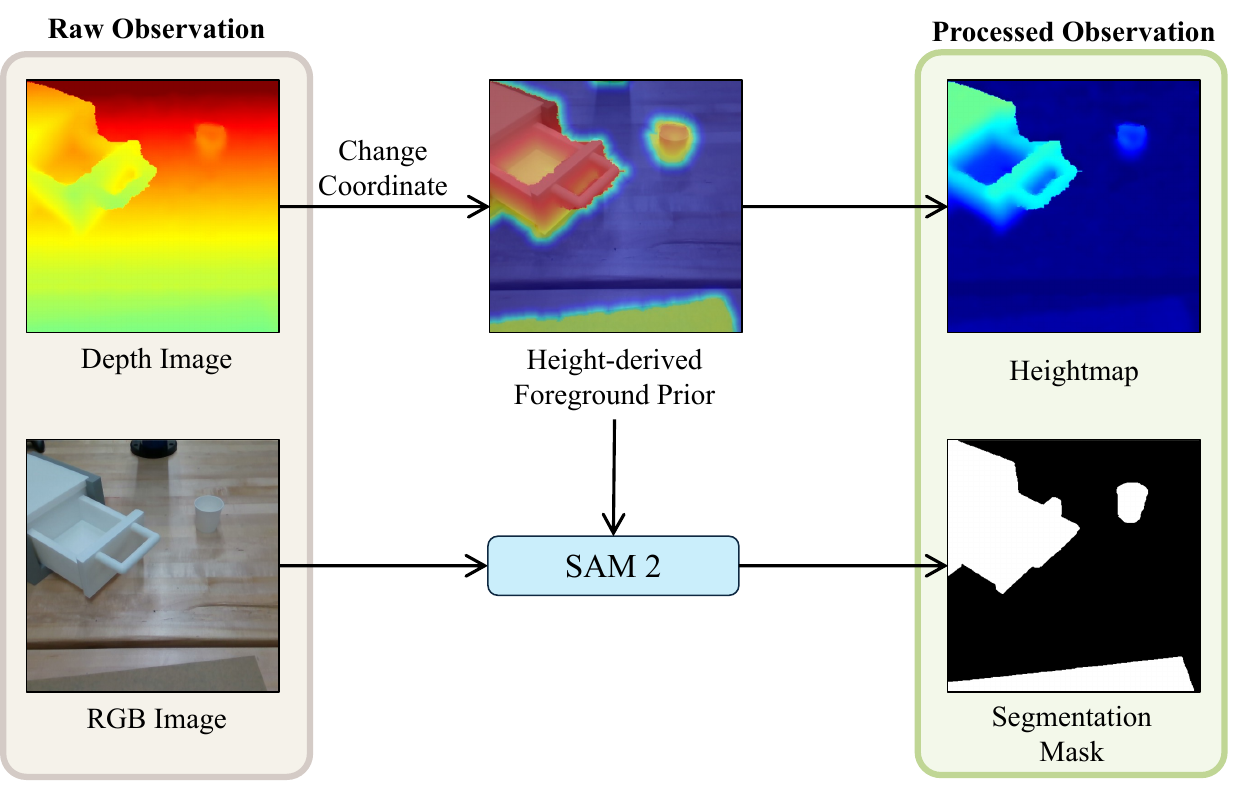}
  \caption{\textbf{Real-world observation preprocessing.}
    Depth is projected into a top-down heightmap, and a height-derived foreground prior prompts SAM~2 to segment the RGB image.
    The heightmap is clipped to $[0,0.3]$\,m above the table and rescaled to an $8$-bit image, then stacked with the binary segmentation mask to form the two-channel policy input.}
  \label{fig:real_obs}
\end{figure}

%% file: appendix/evaluation_sim2real.tex
\subsection{Training and Deployment Protocol}
\label{appendix:sim2real_training}

For real-world deployment, we use heightmap and segmentation-mask observations (Appendix~\ref{appendix:real_obs_preprocess}), similar to~\citet{lin2025constraintpreserving}.
We run Stage~1 and Stage~2 symmetry discovery in simulation on this modality using the procedure in Appendix~\ref{appendix:implementation}, then freeze the resulting perceptual autoencoder and the observation and action symmetry modules.

We train a separate EBC-\method{} policy entirely in simulation using these frozen transformations for frame averaging.
The base policy architecture follows Appendix~\ref{appendix:ebc_training}, with its input encoder adapted to the two-channel input, and the BC objective remains Eq.~\eqref{eq:ebc_loss}.
Relative to the RGB experiments, the changes are the observation modality, the corresponding input adaptation, and the observation--action transformations learned for this modality.

The resulting policy is deployed \emph{zero-shot} on the real robot using the hardware and preprocessing in Appendix~\ref{appendix:real_env_description}.
We report simulation performance in Appendix~\ref{appendix:sim2real_success_rate} and real-world outcomes in Appendix~\ref{appendix:real_eval}.

\subsection{Simulation Results with Heightmap and Segmentation-Mask Observations}
\label{appendix:sim2real_success_rate}

We evaluate behavior-cloning policies in simulation using the same heightmap and segmentation-mask observation modality used at real-world deployment (Appendix~\ref{appendix:real_obs_preprocess}).
This isolates the contribution of the discovered symmetry from the observation-modality change: every policy in Table~\ref{tab:real_obs_results} consumes the same heightmap and segmentation-mask input, and the only difference between rows is whether the discovered observation--action transformations $(\phi_o,\phi_a)$ are used for frame-averaged BC (EBC-\method{}, Appendix~\ref{appendix:frame_averaging}) and, if so, with which finite cyclic subgroup $H=C_n$.

Each cell reports mean and standard deviation of episodic success rate over three seeds and the same evaluation protocol used in Appendix~\ref{appendix:ebc_eval}.

\paragraph{Results.}
Table~\ref{tab:real_obs_results} shows that the trends observed under RGB observations carry over to the heightmap and segmentation-mask modality: increasing $|H|$ from $C_2$ to $C_{16}$ monotonically raises the average success rate from $75.31\%$ to $81.97\%$, and EBC-\method{} with $C_{16}$ attains the best average and the best per-task numbers on Drawer Opening and Cube Dropping.
The vanilla BC baseline retains the lead on Drawer Closing, where the task is already largely saturated ($\approx 91\%$) and frame averaging offers little remaining headroom; EBC-\method{} with $H=C_8$ marginally edges out the other rows on Hammer Lifting.
Overall, all EBC-\method{} configurations with $H\in\{C_4,C_8,C_{16}\}$ outperform the BC baseline on average, supporting the claim that \method{} also discovers a usable $SO(2)$ symmetry under the deployment-time heightmap and segmentation-mask modality.

\begin{table}[!htbp]
  \centering
  \caption{
    \textbf{Simulated success rate with heightmap and segmentation-mask observations.}
    Episodic success rate (mean $\pm$ std over three seeds) on the heightmap and segmentation-mask input used at real-world deployment.
    We compare vanilla BC against EBC-\method{} with $H=C_n$ for $n\in\{2,4,8,16\}$.
    \textbf{Bold} indicates the best entry per column.
  }
  \label{tab:real_obs_results}
  \begin{adjustbox}{width=1.0\linewidth}
    \begin{tabular}{l|c|cccc|c}
      \toprule
      Method                                            & $H$      & \textbf{Drawer Opening}      & \textbf{Drawer Closing}      & \textbf{Cube Dropping}       & \textbf{Hammer Lifting}      & \textbf{Average}             \\ \midrule
      \multicolumn{1}{l|}{\multirow{4}{*}{EBC-\method}} & $C_{2}$  & $0.5713 \pm 0.0389$          & $0.8843 \pm 0.0416$          & $0.7453 \pm 0.1427$          & $0.8113 \pm 0.0413$          & $0.7531 \pm 0.1388$          \\
      \multicolumn{1}{l|}{}                             & $C_{4}$  & $0.5927 \pm 0.0208$          & $0.8971 \pm 0.0539$          & $0.8046 \pm 0.0351$          & $0.8331 \pm 0.0804$          & $0.7819 \pm 0.1275$          \\
      \multicolumn{1}{l|}{}                             & $C_{8}$  & $0.6118 \pm 0.0203$          & $0.9070 \pm 0.0359$          & $0.8271 \pm 0.0216$          & $\mathbf{0.8603 \pm 0.0631}$ & $0.8016 \pm 0.1229$          \\
      \multicolumn{1}{l|}{}                             & $C_{16}$ & $\mathbf{0.6694 \pm 0.0794}$ & $0.9050 \pm 0.0380$          & $\mathbf{0.8536 \pm 0.0423}$ & $0.8507 \pm 0.0253$          & $\mathbf{0.8197 \pm 0.1028}$ \\ \midrule
      BC                                                & -        & $0.6100 \pm 0.0361$          & $\mathbf{0.9159 \pm 0.0075}$ & $0.7266 \pm 0.0825$          & $0.8566 \pm 0.0903$          & $0.7772 \pm 0.1351$          \\ \bottomrule
    \end{tabular}
  \end{adjustbox}
\end{table}

%% file: appendix/real_world_validation.tex
\subsection{Real-World Validation}
\label{appendix:real_eval}

We evaluate the policy on the real robot using the hardware and preprocessing in Appendix~\ref{appendix:real_env_description} and the training and deployment protocol in Appendix~\ref{appendix:sim2real_training}.
The four task setups are shown in Figure~\ref{fig:real_tasks_obs}(a--d) of the main paper.
Table~\ref{tab:real_results} reports the success counts over 20 trials per task.

\begin{table}[!htbp]
  \centering
  \caption{Real-world task success counts over 20 trials per task.}
  \label{tab:real_results}
  \begin{tabular}{l|c}
    \toprule
    Task           & Successful trials \\
    \midrule
    Drawer Opening & 13/20             \\
    Drawer Closing & 19/20             \\
    Cube Dropping  & 9/20              \\
    Hammer Lifting & 16/20             \\
    \bottomrule
  \end{tabular}
\end{table}

%% file: appendix/sim2real_symmetry_vis.tex
\subsection{Qualitative Visualization of Learned Observation--Action Transformations}
\label{appendix:sim2real_symmetry_vis}

Figure~\ref{fig:vis_symmetry_real_all_envs} complements the policy results with visualizations of the coupled observation--action transformations for simulated heightmap and segmentation-mask observations. We show each $g\in C_4$ for all four manipulation tasks.
In every environment, $\phi_o(g,o)$ transforms the top-down workspace and its accompanying segmentation mask jointly so that the target object remains correctly localized in the transformed observation.
On the action side, $\phi_a(g,a)$ adjusts the normalized planar commands $(\Delta x,\Delta y)$ jointly with the observation transformation. As in the RGB experiments, $g$ specifies a latent rotation whose effect in observation and action coordinates is learned.

\begin{figure}[htbp]
  \centering
  \includegraphics[width=\linewidth]{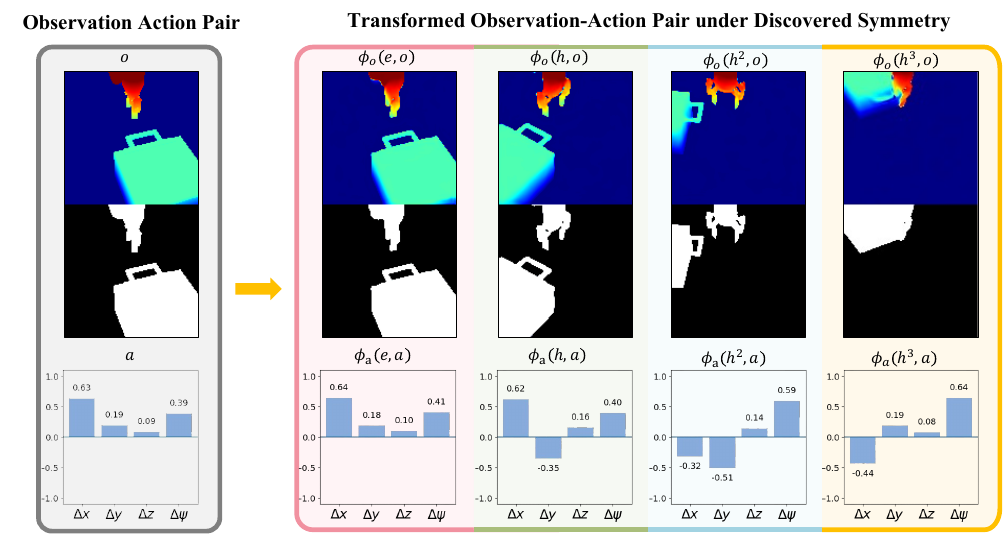}\\[2pt]
  \small (a) Drawer Opening\\[6pt]
  \includegraphics[width=\linewidth]{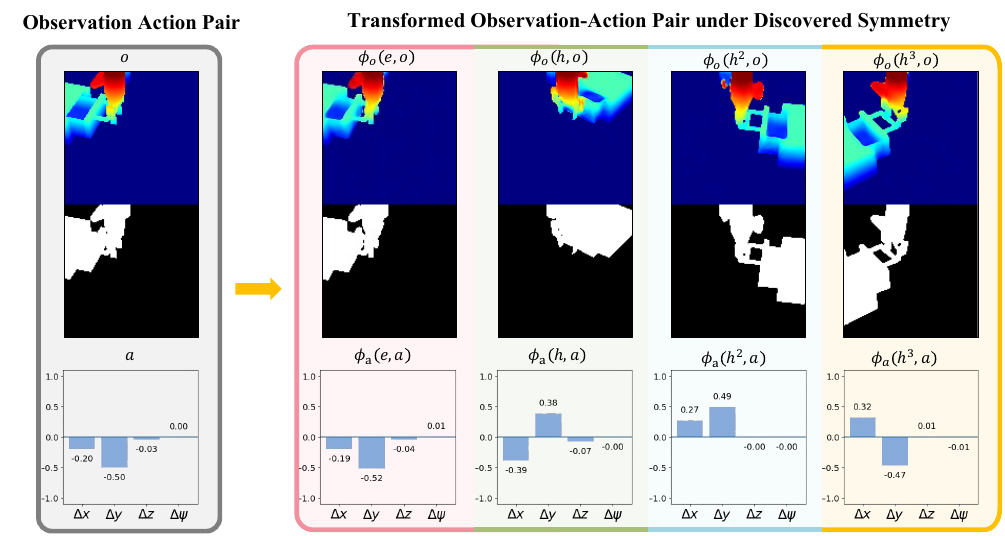}\\[2pt]
  \small (b) Drawer Closing
  \caption{
    \textbf{Learned transformations with simulated heightmap and segmentation-mask observations.}
    Each task panel shows one fixed original pair $(o,a)$ on the left and four transformed pairs on the right for $g\in C_4=\{e,h,h^2,h^3\}$, with latent angles $0^\circ,90^\circ,180^\circ,270^\circ$, respectively.
    Within each observation, the heightmap appears above the segmentation mask; action bars below show normalized command components on a $[-1,1]$ scale.
    The same $g$ transforms both observation channels and the action jointly; decoded transformations need not be geometric rotations by the latent angles.
    Panels (c)--(d) continue on the next page.
  }
  \label{fig:vis_symmetry_real_all_envs}
\end{figure}

\begin{figure}[t!]\ContinuedFloat
  \centering
  \includegraphics[width=\linewidth]{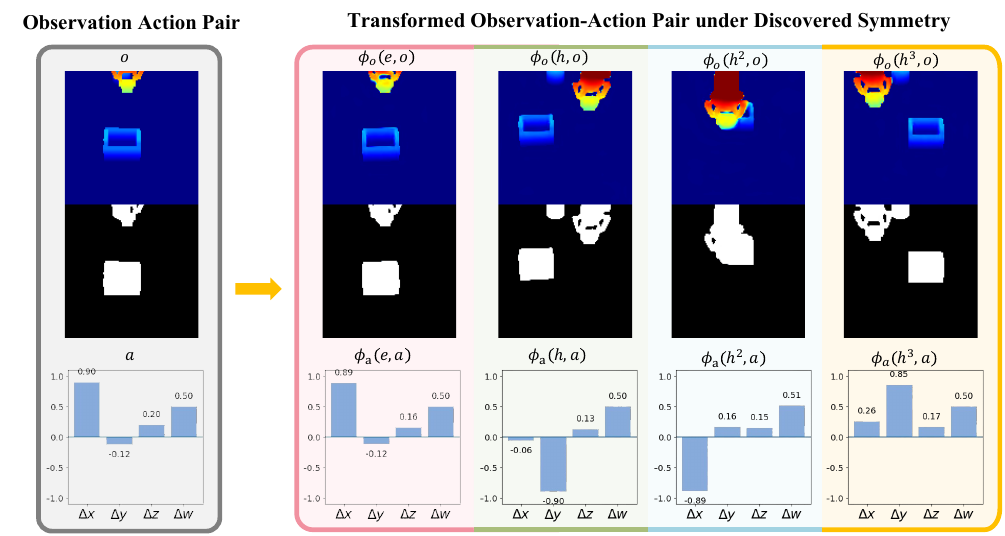}\\[2pt]
  \small (c) Cube Dropping\\[6pt]
  \includegraphics[width=\linewidth]{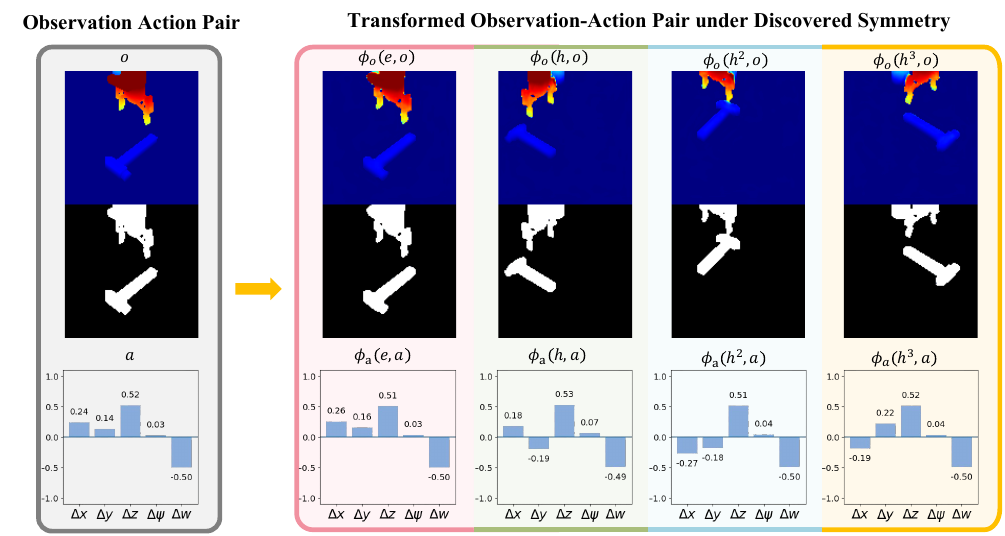}\\[2pt]
  \small (d) Hammer Lifting
  \caption{\textbf{Learned transformations with simulated heightmap and segmentation-mask observations (continued).}
    Panels (c)--(d) show Cube Dropping and Hammer Lifting, with the original pair on the left and four transformed pairs on the right.
    Columns $e,h,h^2,h^3$ correspond to latent angles $0^\circ,90^\circ,180^\circ,270^\circ$, respectively; decoded transformations need not be geometric rotations by these angles.
    Each observation shows the heightmap above the segmentation mask, with normalized action components on a $[-1,1]$ scale below.}
\end{figure}